\documentclass[12pt]{article}

\usepackage{tabularx}
\usepackage{arydshln}

\usepackage{amsmath, amsthm, amssymb}
\usepackage{algorithm}
\usepackage{algorithmic}
\usepackage{enumerate}
\usepackage{nccmath}
\usepackage{geometry}
\usepackage{bm}
\usepackage{multirow}

\newcommand{\gry}{incongruity }
\newcommand{\gryp}{incongruity. }
\newcommand{\grym}{incongruity, }

\newcommand{\eps}{\varepsilon}
\newcommand{\ph}{\varphi}
\newcommand{\de}{decision }
\newcommand{\des}{decisions }

\newcommand{\bydef}{\overset{\operatorname{def}}{=} } 

\newcommand{\bb}{\mathbb}

\usepackage{enumerate}

\usepackage{array}
\newcolumntype{L}{>{\arraybackslash}m{7cm}}

\usepackage{amssymb}

\usepackage{tcolorbox}

\usepackage{wrapfig}
\usepackage{geometry}
\newtheorem{definition}{Definition}
\newtheorem{statement}{Statement}
\newtheorem{theorem}{Theorem}

\usepackage{xurl} 
\usepackage[colorlinks,allcolors=blue]{hyperref} 

\title{Logic of Machine Learning }
\author{Marina Sapir }
\date{}

\begin{document}

\maketitle

\begin{abstract}

	The main question is:  why and how can we ever predict based on a finite sample? The question is not answered by statistical learning theory. Here, I suggest that prediction  requires belief in ``predictability'' of the underlying dependence, and  learning involves search for a hypothesis where these beliefs  are violated the least given the observations. The measure of these violations  (``errors'') for given data, hypothesis and particular type of predictability beliefs   is formalized as concept of \gry in 
 modal Logic of Observations and Hypotheses (LOH).  I show on examples of many popular textbook learners (from hierarchical clustering to k-NN and  SVM)  that each of them minimizes its own version of  \gryp
 In addition, the concept of \gry is  shown to be flexible enough  for formalization of some important data analysis problems, not considered as part of ML.

	\end{abstract}
\section*{Introduction}

ML is usually associated with making predictions to improve decisions. Of course,  the  future is unknown, so prediction got to be difficult.  But we can't know much about the past either. 

An applied ML scientist (practitioner) is aware that   nothing is known for sure about the reality we model. For example,  the features may not define completely  the feedback we are trying to model, there is uncertainty in measurements, random mis-classification and so on. Objects with the same features may have different feedback, and the same object evaluated twice may have different features or even feedback.

This is not a bad luck, but a inevitability.  Indeed, if there is no exact theory explaining the phenomenon we are trying to predict, we do not know what it depends on or how to measure it. If there is a theory, one does not need  ML.  ML deals with real raw life, not an abstraction. 

More the over, the time is critical  when it concerns prediction based decision making. It means the less time we spend on accumulating the data for predictions, and the less data we use for prediction, the better. So,  data shortage is not a bug, it is a feature of ML. 

 Thus to predict future we need to model nondeterministic dependence with as little data as possible. 
 
 The first question should be:   what does it mean to model a nondeterministic  underlying dependence, and how do we, actually, do it?

The answer proposed here is that modeling is possible because (i) we implicitly rely on  ``predictability'' of the underlying dependence: close or identical data points shall correspond to close feedback. Violation of this principle is called here ``incongruity''.  And (ii) we search for the least incongruent model of the dependence.

A Logic of Observations and Hypotheses  as well as concept of ``proper aggregation''  are  introduced here to  formalize the idea of \gry of a hypothesis given the predictability beliefs and data.

The main conjecture of this work is that each learner has its own version of \grym and its loss criterion evaluates this \gry for a given hypothesis and the training set. As a part of this main conjecture, I outlined general steps each learner performs, exposing  inner similarity of  diverse learners from $k$-NN to  $K$ Means. 

To the best of my knowledge, here it is shown for the first time  that  the large variety of  learners can be described in common terms and that they have common logical justification.

 \section{Traditional  views on ML} 
 
 Here I describe a ``a naive''  idea of ML and the issues with this idea. I  pose the questions the theory needs to answer.  Then I present the only commonly accepted theory of automatic learning and show that it does not really answer these questions.  

\subsection{Prediction problem} 

Denote $\Omega$ the set of real life objects of interest.  For example, this  may be patients with skin cancer, or bank clients or  engine failures. There is a  hidden  essential quality of the objects we would like to find out (may be, a diagnosis or prognosis).
Some properties (\textbf{features}) of the objects  $\Omega$   can be always evaluated and numerically expressed. Some of them are expected to be relevant to the hidden property.  Suppose, there are $n$ such features.  Denote $X \in R^n$ domain of  feature vectors for objects in $\Omega$.   The hidden essential quality also has numerical expression from  domain $Y \in R.$ The value of the hidden essence in a given object  is called ``\textbf{feedback}''.  We assume there is an ``\textbf{underlying dependence}''  $\varphi: X \rightarrow Y$ between feature vectors  and the feedback. Yet, we can not assume that the dependence is deterministic.

The information about the underlying dependence $\ph$  is given as (imprecise) observations about values of  feedback  in certain data points and can be recorded as set of  formulas 
$\{\ph(x) \approx y\}$ or as  set of tuples $\{\langle   x, y \rangle\}.$  

The set of recorded observations is called \textbf{training set}. 
In a Prediction problem, the goal is  to find a function $f: X \rightarrow Y$, which is ``close enough''  to the underlying dependence $\varphi$, in the sense that the probability of large errors on  future observations of objects in $\Omega$  is low enough: 
$$P[| \varphi(x) -  f(x)| > \delta] < \epsilon.$$

The Prediction problem is ill posed: generally, knowing a finite set of observations   with non repeating  data points does not imply anything about expected values of feedback $\ph(x)$ in the same data points, let alone in other points. Strictly speaking, the givens and goal  of the problem are not related. 



Any theory of machine learning needs to answer at least two \textbf{fundamental  questions:}
\begin{enumerate}
	\item Q1:  What shall be done with the training set for learning ?
	\item Q2: When and why can a decision  predict well enough?
\end{enumerate}
	
In the next subsection, I show how  statistical  learning  theory understand machine learning.

\subsection{Statistical Learning Theory Approach}

Statistical Learning  (SL) theory is  the only commonly accepted theory of ML. The most popular version of this  theory is also called ``VC-theory'', because VC-dimension plays important role here.

There is some confusion about the term ``learner'' in learning theory textbooks. For example,  in \cite{Shalev},  the term is understood as a procedure for solving a problem  in a finite number of steps -  when actual procedures are discussed. But in chapters   talking about statistical learning theory (PAC learning, VC-theory) the  terms  ``learner''  and "algorithm" mean ``functional that takes in a finite training set  $S$
and outputs a function $h: X \rightarrow Y $ '' \cite{CPAClearning}: 
``VC-theory did not impose any requirement on the learners actually being implementable by algorithms''. For disambiguation, talking about SL, I will use the term ``learning functional''.

Thus, strictly  speaking, statistical learning theory does not talk about  the main subject of this work, the learning algorithms.  Still, I will describe its main results here because of its unique importance.  

Only binary labels  $Y$ are considered here.

Denote $D$ the distribution on $X  \times Y$ from which the training $S$ set is drawn. Denote 
$$L(f, S) =\sum_{S} I(f(x_i) \neq y_i),$$
empiric risk of a a hypothesis $f: \chi \rightarrow Y$ on the training set $S,$ 
$$L(f, D) = P[f(x) \neq y]| \langle x, y \rangle   \sim D],  $$ a generalized risk, or error on the general population, 
$$L(H, D) = \min_{f \in H} L(f, D), $$
optimal generalized risk in the class of hypotheses $H.$

The hypothesis  $h = \alpha(S)$ output by a learning functional  $\alpha$  on a given training set  will be called a\textbf{ decision} of this functional on $S.$

For a given $\eps > 0$  a hypothesis $h$ is \textbf{approximately correct} for the class of functions $H$  if $L(h, D) < L(H, D) +  \eps$. Denote $ A_\eps (h, D)$ this property
$$ A_\eps (h, D) \bydef \left( L(h, D) < L(H, D) +  \eps\right).$$

As the work \cite{CPAClearning} clarifies, the approximately correct for the class $H$ \de  $h$ does not have to belong to the class $H, $ and it   does not have to be in such form that one could use it to calculate the function $h$ on any arguments.

Instead of considering a single \de on a given training set $S$,  statistical learning is interested in all the  \des obtained on  different training sets with a given  lower bound $m$ on their size.  Denote $ Z(\alpha, m, D)$  set of all  the \des  by the  learning functional  $\alpha$ on all the training sets of  the size $m$ or larger randomly generated by the same distribution $D$.

For a given constant $\delta$, if 
$$\underset{h \in Z(\alpha, m, D)}{P}[ A_\epsilon(h, D)] > 1 - \delta,$$ 
every \de   $h \in Z(\alpha, m, D)$  is called\textbf{ probably approximately correct}.  Denote this property 
$$ B( \alpha, m, D, \epsilon, \delta) \bydef  \left( \underset{h \in Z(\alpha, m, D)}{P}[ A_\epsilon(h, D)]  > 1 -  \delta\right).$$ 

 For a given class $H$, a learning functional  $\alpha$ is called \textbf{successful learner}, if    for every $\eps, \delta, $   regardless of distribution $D,$  for a large enough $m(\eps, \delta)$  predicate $B( \alpha, m(\epsilon, \delta), D, \epsilon, \delta)$ is true.   

The main focus in the theory is classes of PAC-learnable  functions.  The class $H$ is called \textbf{PAC (probably approximately correct) learnable} if there exists a successful learning functional for this class.

``Our current goal is to figure out which classes $H$ are PAC learnable, and to
characterize exactly the sample complexity of learning a given hypothesis class.'' \cite{Shalev} (sample complexity is politically correct name for minimal size of the training sample.)  

\textbf{The main result of the theory is that classes of functions which have finite VC dimension and only such classes are  PAC learnable.} The theory claims there is a function $m(\epsilon, \delta, VCd(H) )$, which gives lower bound of the parameter $m$ in the set  $Z(\alpha, m, D)$ of PAC decisions for PAC-learnable classes. Also, it states that if class is PAC learnable, then a functional which picks a functions from the class $H$ with minimal empiric risk is a successful PAC learner.

Here are some issues with SL.
\begin{enumerate}
	\item \textbf{The theory solves an irrelevant problem}. The main result  of the theory 
	\begin{itemize}
	\item expresses required size of the training set through  VC dimension of the function class $H$, yet the decisions may be outside of the class $H$; the relevance of the class  $H$ is not clear; 
	
	\item makes statement about an arbitrary  distribution, yet  most of the distributions  are of no interest because they do not support the  existence of  an underlying dependence;  it should be easier to find a dependence when it exists.   
	
		\item evaluates  probability of having ``approximately correct'' hypothesis  in   $Z(\alpha, m, D).$  This set of hypotheses is pure speculation: there is only one training set. But even we imagine $Z(\alpha, m, D)$,  infinitely many hypotheses in this set are obtained on training sets arbitrary larger than $m$ and non is obtained on sets  smaller than $m$. Therefore even if probability of failure on  $Z(\alpha, m, D)$ is low, it implies  nothing about likelihood of failure on a given training sample of the size $m$. 
		
		\item recommends empiric risk minimizing functional of a class $H$ as a successful PAC learner, but, if the class $H$ is infinite, there may not be an algorithm implementing  this functional in a finite number of steps.  
		
		\item talks about relative generalized risk, while practical applications are interested only in absolute generalized risk.
	\end{itemize}

\item  \textbf{ Indefinitely  increasing training set size contradicts the idea of learning}. 
\begin{itemize}
		\item  \textbf{The training set does not increase.}
	``Intuitively, it seems reasonable to request that a learning algorithm, when presented more and more training examples, should eventually “converge” to an optimal solution." \cite{StatTheory} There is nothing intuitive or reasonable about such a request, because training set is fixed.	
	
\item \textbf{The bounds of the sizes of the training sets are unreasonable} The theory gives upper bounds for the desired training sample size. These sizes are grotesque. There are no training sets of such sizes for most of ML problems. 
 
 	\item\textbf{ Large training sets are counterproductive. } Timing of decisions which are supposed to be made on the basis of ML  is critical, therefore smaller training sets are preferable. Large training sets as the theory requires   would  describe the master distribution  in very fine details, making ML pointless.  People do ML exactly because it allows one to  compensate  for lack of accumulated knowledge. 

\end{itemize}  
 
\item \textbf{SL ignores both critical questions}
\begin{itemize}
\item The theory deals with  learning  functionals and not algorithms, so it can not answer the first critical question (``What shall be done with the training set for learning ?''). 

\item Being interested only in relative loss, SL can  not  answer the second critical question: ``When and why can learner's  decision  predict well enough?''   
\end{itemize}

  \end{enumerate} 

Theoreticians usually do not dwell on these issues. But it does not mean that they do not notice them. 

One of fathers of SL,  V. Vapnik formulated the justification  for SL theory  in the most direct way \cite{VapnikBook}:

\begin{quotation}
	\textit{	Why do we need an asymptotic theory $\langle \cdots \rangle$ if the goal is to construct algorithms from a limited number of observations?
		The answer is as follows:
		To construct any theory one has to use some concepts in terms of which the theory is developed $\langle \cdots \rangle.$ }
\end{quotation}

I other words, SL has to use the statistics toolkit. Statistics  has laws of large numbers, and it is what one uses for deduction  in statistical learning theory.

Perhaps, lack of the suitable apparatus to understand  the true problem with fixed finite data  and un-quantifiable  uncertainty is the root of this divergence between the theory and needs of the applications. 

The theory so remote from real life applications  cannot help  practitioners, does not answer the most common questions they raise \cite{Papaya}. 

Therefore, there is a  need in a learning theory which would make sense of   the actual practice of ML. I show below that, based on popular learners,  learning is possible not because of  ever increasing training set, but because  the  underlying dependence is expected to be ``predictable'' or ``congruent''.

\section{Predictability} \label{Raugh}

I am not trying, vainly, to forecast  accuracy of a decision on the general population. Rather, I want  to answer the first critical question: what do we do with the training  set in applied ML, and what is the logic behind these manipulations?

\subsection{Existing Approaches to Logics of Uncertainty} 

There are plenty of well established logical approaches to reasoning under uncertainty,  as well as to study of  nondeterministic dependencies  and modeling inconsistent data.  

Modal logics are introduced to take into account some subjectivity and uncertainty.  Yet, they would not tolerate inconsistency.  

Fuzzy logic  \cite{Fuzzy} and Subjective Logic \cite{Subjective} would not help either, because they  assume there is an objective omnipotent observer, who   can quantity degrees of certainty or belief about  given statements. The systematic review \cite{ReviewUncertainty} describes various approaches to reasoning under uncertainty as ways to quantify and exactly measure uncertainty of statements and sets of statements. Exactness about uncertainty of empirical observations appears to be a  contradiction in terms. 

A typical approach to resolve inconsistency of knowledge is to assign some kind of ``certainty'' or ``preference'' for each formula, and then select a the most preferable (``probable", ``certain'', ``reliable'') subset of consistent formulas. One of the first works of this type was \cite{Inconsistency}. The main idea there is  to assign reliability to each statement and remove the least reliable ones to avoid a contradiction. Here are the main issues with this approach:
\begin{enumerate}
	\item When we are not certain about the knowledge, we can not be certain about comparative reliability of it. 
	\item  For a nondeterministic dependence,  contradictory observations are the rule, not an exception. Together they create more complete picture of reality than any non-contradictory subset.  
	\item In case of ML  contradictions between a  hypothesis and  noisy observations shall be present always: exact fit of noise is not desirable.  It means, excluding inconsistency is not an option.
\end{enumerate} 

Nondeterministic logics \cite{Nondeterministic} and logical operations with nondeterministic tables are introduced to derive logical functions from exact but incomplete data, which is different from the situation of inconsistent  observations or inconsistency between hypothesis and the training set.

There are several approaches to describe logic of learning. For example, \cite{Epistemic} considers asymptotic learning: precise observations are presented indefinitely, the ``nature'' has in mind particular function from the given class, and the learner has to chose the correct  hypothesis. Unfortunately, in applied ML,  all the good assumptions about this learning idea are false: the training set is finite and too small, observations are known to be tentative, and whatever nature has in mind,  is far from the selected class of functions or, rather,  is not a function of given features at all. 

The researchers already noticed that logic where all formulas have truth values does not describe certain types of logical reasoning \cite{NoTruth}, particularly legal reasoning about existing norms. The logics with modalities like ``it ought to be", ``you ought to do" are introduced, but to the best of my knowledge, epistemic modalities which can not be associated with  truth values were not explored.

\subsection{Informal description of the new approach}

Let us go back to the original prediction problem: Given imprecise observations of non-deterministic dependence $\ph$, to find a function $f$ to model $\ph$ and predict  its values on new data points. I noticed that the problem is incorrect, can not be solved as is.

In real life, the solution of a ML problem is possible, because we believe   in  ``predictability'' of  $\ph$: the dependence has to have similar values in close data points. If we do not believe it,   there is no problem. 
 
    And we, usually,  have good reasons for predictability belief.  ML problems do not appear from nowhere. They are thoroughly prepared by the same specialists who solve them.  Preparation includes posing  a meaningful  question and developing the features to be predictive. 
    
    From philosophy point of view,  predictability beliefs are founded  on  the  \textbf{fundamental belief} in  inner congruence, ``consistency''  of the reality: the world does not, usually, change sharply. Otherwise, the  prediction would not be possible and we could not exist as animals, let alone humans: the brain is an organ of prediction in all animals \cite{Brain}. The fundamental  belief   leads us further to prefer  models which appear to be more ``consistent''  with the the  available data: if the model was more ``agreeable'' in the past, we have expectations that it will continue to be. Thus, we usually assume that optimal agreement with available data  is a desirable property of the model. 
    
      Instead of searching for  the model which will work good in the future -  which is impossible -  we search for  a  model which works  on accumulated data both well enough and  the best among others. And this problem is, often,  tractable, at least approximately. 
    
   As predictability may take different forms, depending on the problems, so does \gryp   To understand the process of learning  we would need a general concept of incongruity.
   
   In the  simplest case, for each observation $\langle x, y \rangle$ , we take a hypothetical instance $h(x) = y_1$ where the value of the underlying  dependence $\ph$ are evaluated  in the same point $x$, and measure how the values $y_1, y$ are different. 
   
In general, the concept of \gry requires (i) identifying associations between hypothetical instances and observations  and (2) evaluating the disagreement, or deviation,  for each such associated pair. 

This will give us the set of deviations for each hypothesis. So, on top of this,  one would need to aggregate these deviations, so each hypothesis can be characterized by a single number.

One may view  ``deviation'' as a fuzzy measure  of contradiction. If $y_1 = y$ in the above example, there is no contradiction between the observation and the associated hypothetical case.  Yet, the values $y, y_1$ are not expected to be identical. The concept of logical contradiction is meaningless for a nondeterministic function and its model. So, deviations evaluate degree, to which the desired ``consistency'' is violated.  Then value of  \gry is a fuzzy measure replacing rigid concept  of inconsistency of set of formulas. 

The conjecture of this work is that every learner has its own   concept of  \gry  as a loss criterion to compare  hypotheses  and derive the decision. The conjecture will be corroborated on examples of many learners.  

To avoid inconsistency in  reasoning about nondeterministic dependencies,  I introduce  subjective modalities of  perceptions or assumptions. These modalities allow contradictions. For example, different subjects (or even the same subject)  can observe things differently in different times.  This makes  the concept of inconsistency of our knowledge irrelevant.

\section{Formal definition of the approach} 

Modal logic of observations and hypotheses (LOH)  formalizes reasoning about  predictability and  deviations from predictability belief.   The   instances are presented as first order modal formulas.  The second order relations on the first order formulas will be used to describe predictability.

\subsection{Logic of observations and hypotheses (LOH) }

The first order signature has four  sorts:

\begin{table}[H]
	\begin{center}
		\caption{Sorts of LOH}
		\label{tab:table1}
		\begin{tabular}{| l | l | l | l| l | }
			\hline
			& Sort  & Meaning  & Variables  & Constants \\ \hline
			1&  $\bb{X}$ & finite numeric set from $\bb{R}^n$  & $x, x_1, x_2, \ldots$ &\\
			2 & $\bb{Y}$ & finite numeric set  & $y, y_1, y_2, \ldots$ & \\ 
			3 & $\Theta$ &   symbols of modalities & $s, s_1, s_2, \ldots$ & \\
			4 & $\bb{R} $ &  real numbers & $r, r_1, r_2 \ldots $ & $a, b, a_1, a_2, a_3, \ldots,$ \\
			\hline
		\end{tabular}
	\end{center}
\end{table}

We assume the domains of the sorts $\bb{X}$ and  $\bb{Y}$ are subsets of some metric spaces. For example,  the  set $Y$ with two values $\{0, 1\}$ may be considered a metric space with the norm
\[
\|y_1 - y_2\| = 
\begin{cases}
    0, & \text{if }  y_1 = y_2\\
    1, & \text{otherwise.}\\
\end{cases} 
\]
It is obvious that the norm satisfies the axioms of metric spaces.

We will consider the next symbols of modalities $\{\approx, \approx_1,  \ldots,  \asymp, \asymp_1, \ldots\}.$ The symbol $\asymp$ is interpreted as \textit{Assume that}, it indicates the hypothetical instances.  The modalities $\approx$ are interpreted as \textit{It appears that} and indicate observations. The index in a notation of  modality are used to separate   groups of hypothetical instances and groups of observations with different context. The number of  modalities may be different between interpretations.

There is one dedicated first order unary  functional symbol $\ph:  \bb{X} \rightarrow \bb{Y},$ which denotes the   underlying dependence,  and  the only first order relation $=$ with standard interpretation.  

Interpretations of this logic may have  optional  other sorts,   first order functional symbols and operations, which will be specified in formalization of each problem as needed. 

All the first order formulas have the form $\square (  \ph(x) = y ) $
where  $\square \in \Theta$ is one of the modalities. The formula $\approx(\ph(x), y)$ corresponds to  an observation $\langle x, y \rangle$, the formula $\asymp(\ph(x) =  y)$ corresponds to the hypothetical instance $\langle x, y \rangle.$

The formulas are not assigned truth values:  they reflect subjective, uncertain knowledge. Nothing can be deduced from these formulas. There is no logical connectives, no first order inference. 

The set of all first order formulas of a LOH model $\mathcal{M}$ will be denoted as $\Upsilon(\mathcal{M}). $ The set $\Upsilon(\mathcal{M})$ is the domain of the second order functions and predicates. The variables $\alpha, \alpha_1, \ldots$ will denote formulas from $\Upsilon(\mathcal{M}).$

\begin{table}[h!]
	\begin{center}
		\caption{Second order function symbols}
		\label{tab:Functions}
		
		\begin{tabular}{|l | l | l | l | l |  }
			\hline
			& Symbol   & Arity & Sorts  & Semantic   \\ \hline
			1 & $\bm{x}$  & 1  & $\Upsilon(\mathcal{M})  \rightarrow \bb{X}$ &  the first variable   \\
			2 & $\bm{y}$  & 1  & $\Upsilon(\mathcal{M})  \rightarrow \bb{Y}$ & the second variable    \\
			3 &  $\bm{s}$ & 1  &  $\Upsilon(\mathcal{M})  \rightarrow \Theta$  & modality symbol\\
			4 & $\rho_x$ & 2  & $\Upsilon(\mathcal{M}) \times \Upsilon(\mathcal{M}) \rightarrow \bb{R}$ &  distance between first variables  \\
			5 & $\rho_y$ & 2  & $ \Upsilon(\mathcal{M})  \times  \Upsilon(\mathcal{M})   \rightarrow \bb{R}$ &  distance between second variables \\
			\hline
		\end{tabular} 
	\end{center}
\end{table}
\bigskip
\hskip2pt

The second order  functions  are defined by the  axiom
$$ \mathcal{A}_1: \;  \forall \alpha  \; \Big( \bm{s} (\alpha) \; \big( \; \ph( \bm{x}( \alpha) ) = \bm{y}(\alpha) \;\big) \Big) = \alpha.$$

$$\mathcal{A}_2: \forall \alpha_1, \forall \alpha_2 \; \rho_x(\alpha_1, \alpha_2) = \|\bm{x}(\alpha_1) - \bm{x}(\alpha_2)\|.$$

$$\mathcal{A}_2: \forall \alpha_1, \forall \alpha_2 \; \rho_y(\alpha_1, \alpha_2) = \|\bm{y}(\alpha_1) - \bm{y}(\alpha_2)\|.$$

The second order formulas will have 
\begin{itemize}
	\item relations $\le, < , > , \ge, =,$ on real numbers, 
	\item regular logical connectives ($\vee, \&$)
	\item   real valued constants. 
\end{itemize}

\subsection{Incongruity Concept}

Informally,  \gry means that for associated pairs of formulas, when arguments of the underlying dependence in them   are ``close'', so are their  feedback. 

\subsubsection{Main Definitions}

\begin{definition}[Collision condition]
	Any second order LOH  statement with two free variables  over  first order formulas  will be called  an \textbf{collision} condition. 
\end{definition}

Figuratively speaking, an collision relation  identifies pairs of formulas which potentially can be ``vaguely contradictory''.  For such a pair of formulas, ``deviation'' will determine the degree of  its  ``vague contradiction'', or collision.

\begin{definition}[Deviations] 
	\textbf{Deviation function} is a function $t(r_1, r_2): \bb{R}^+ \times \bb{R}^+ \rightarrow \bb{R}^+ $ isotone by $r_1$ and antitone by $r_2.$ 
  
  For a pair of first order formulas $\alpha_1, \alpha_2$ and a deviation function $t$  their \textbf{deviation degree} is 
  $$ \delta(\alpha_1, \alpha_2) = t(\rho_y(\alpha_1, \alpha_2), \rho_x(\alpha_1, \alpha_2)).$$
\end{definition}

Deviation degrees are often called ``errors''  in ML. 

There may be many aspects of \gryp  This gives rise to the concept of \gry theory. 

\begin{definition}[Incongruity theory]
	A sequence of collision conditions  and corresponding deviation functions $$\{ \langle \pi_i, \; t_i \rangle, i = 1:k\}, k \ge 1$$  is called a \textbf{\gry theory}. 
	Each pair $\langle \pi_i,  \; t_i \rangle$ is called \textbf{\gry aspect}. 
\end{definition}

 For a \gry  theory $T$ the notation   $T  \vdash \langle \alpha_1, \alpha_2 \rangle$  means the theory includes an collision condition $\pi$ such that $  \vdash \pi(\alpha_1, \alpha_2),$ the pair $\alpha_1, \alpha_2$ is in the collision condition $\pi.$
 
 \subsubsection{Example of an \gry theory}
 
This   example of a \gry theory is called \textbf{Point-Wise} \gry theory, $T_{pw}$. It has only one aspect with the collision condition 
 \begin{align}\label{pointwise}
 	\pi_{pw}(\alpha_1, \alpha_2)  = \Big( ( \bm{s}(\alpha_1) \; = \; \asymp)  \; \& \; ( \bm{s}(\alpha_2) \; = \; \approx)  \; \& \; ( \bm{x}(\alpha_1) = \bm{x}(\alpha_2) \Big)
 \end{align}
 and the deviation function $t(r_1, r_2)  = r_1.$
 The function  $t$ is obviously isotone by the the $r_1$. The function does not depend on $r_2$, so for any $r_1, r_2^1, r_2^2 \; t(r_1, r_2^1) = t(r_1, r_2^2).$ Therefore, the condition of antitony by the second variable is not violated. 
 
By definition, hypothetical instance $\alpha_1$  and an observation $\alpha_2$ satisfy the  point-wise collision condition  $\pi_{pw}(\alpha_1, \alpha_2)$ when $\bm{x}(\alpha_1) = \bm{x}(\alpha_2).$ 
 For any pair of first order formulas $\alpha_1, \alpha_2: \; \vdash \pi_{pw}(\alpha_1, \alpha2)$ their deviation $\delta(\alpha_1, \alpha_2) = \rho_y(\alpha_1, \alpha_2).$
 
\subsubsection{Full Model} 
Various models of LOH will have various sets of   first order formulas. The same theory may have different set of deviations depending on the model.  There needs to be an agreement about a LOH model used to evaluate deviations for a given theory, hypothesis  and observations.

\begin{definition}[Full model]
	Given an \gry theory $T = \{\langle \pi_i, t_i \rangle, i = 1:k\}$, hypothesis $h$ and the set of observations $S$, a model $\mathcal{M}$ is a\textbf{ full model} for $h, S, T$ if
	
	\begin{itemize}
		\item $S \subseteq \Upsilon(\mathcal{M})$
	  \item For any hypothetical formula $\alpha$ and for any observation formula $\beta \in \Upsilon(\mathcal{M}) $ if  
	  \[(T \vdash(\alpha, \beta) ) \; \vee  \; ( T \vdash(\beta, \alpha) ) \] then $\alpha \in \Upsilon(\mathcal{M}) $
	  \item For any hypothetical formulas $\alpha_1, \alpha_2$ such that $T \vdash (\alpha_1, \alpha_2)$  the formulas  $\alpha_1, \alpha_2$ are in $ \Upsilon(\mathcal{M}). $
	  \item $\Upsilon(\mathcal{M})$ is the minimal set of formulas, satisfying these conditions. 
	  \end{itemize}
\end{definition}

Given the hypothesis, observations and \gry theory, the definition determines the set of first order formulas of a full model uniquely.

\subsubsection{Example of a  full model}

For example, for the $T_{pw}$ theory, any hypothesis $h$ and the set of observations $S = \{ \beta_1, \ldots, \beta_m\}$ the set of formulas $\Upsilon \big(\mathcal{M}(h, S, T_{pw}) \big)$ will consist of the next two subsets
	\begin{itemize}
		\item $S$
		\item $ \{ \alpha\;  | \; (\bm{s}(\alpha) = \asymp ) \; \& \; (\exists \beta \; (\beta \in S) \; \& \; (\bm{x}(\alpha) = \bm{x}(\beta)\}.$
	\end{itemize}

For a given set of $m$ observations $S$ and a hypothesis $h$ full model of the theory $T_{pw}$ will have all the first order formulas of the observations $S$, and for every observation $\beta \in S$ there will be a hypothetical formula 
$$\Big( \asymp \big( \ph(\bm{x}(\beta)) = h(\bm{x}(\beta) )  \big) \Big).$$ So, there will be exactly $m$ pairs of first order formulas $\alpha_1, \alpha_2$ satisfying the condition $\pi_{pw}.$

\subsubsection{Regularization}

Usually, predictability of a dependence means that it has close values on close data points. 
When a hypothesis is a known differentiable  function, there are ways to evaluate some aspects of its  ``predictability'' independently of data. A good, predictable hypothesis shall be smooth, it has to have uniformly low derivatives.  

This method of including some measure of derivatives in the loss criterion of a learner  is  called ``regularization''.  Regularization is  used  sometimes in addition to   \gry to  measure violation of predictability.

\subsection{Proper aggregation} 

An aggregation operation maps a multiset of real numbers into a real number. 

The operation of aggregation  $TOT(G): 2^{\bb{R}} \rightarrow \bb{R}$ defined on all finite multiset in $ \bb{R}$ is called \emph{proper aggregation}, if it satisfies  \textbf{three axioms}.

\begin{enumerate} 
	\item \emph{\textbf{Monotony }}: For two multisets in $\bb{R}$ If there exists isomorphism $q:$  $G_1 \rightarrow G_2$ such that 
	$$\Big(\forall x\;  (q(x) \ge x) \Rightarrow  TOT(G_2) \ge  TOT(G_1) \Big) \&$$ 
	$$\Big( \forall x \; (q(x) >  x) \; \Rightarrow \; TOT(G_2) >  TOT(G_1) \Big)$$
	\item \textbf{ \emph{Idempotence}}:   $ TOT(G \cup \{TOT(G)\}) = TOT(G). $
	\item \textbf{\emph{Tautology: }}  If $G = \{x\}$ then $TOT(G) = x.$
\end{enumerate}

Some natural properties of proper aggregation follow from the axioms.

\begin{statement}
	Any proper aggregation $TOT(G)$ has the next properties:
	\begin{enumerate}
		\item If the multiset $G$ consists of $n$ identical elements $x$, then $TOT(G) = x.$
		\item $min(G) \leq TOT(G) \leq max(G).$
	\end{enumerate} 
\end{statement}
\begin{proof}
	\begin{enumerate}
		\item Let us prove it by induction by $n = \|G\|.$ It $n = 1,$ it follows from the axiom Tautology. Suppose, the statement is proven for $n = k$.  Then for $n = k + 1$ it follows from the axiom  Idempotence. 
		\item  Let us prove by contradiction. Suppose $$\exists G \forall x \;\: (x \in G)  \Rightarrow (TOT(G) > x).$$
		Denote $n = \|G\|.$ By the previous property,  if the  set $G_1$  consists of $n$ elements  $TOT(G)$ then $TOT(G_1) = TOT(G).$ It contradicts the axiom of  Monotony since every element of $G_1$ is larger than all elements of $G.$ The same way we can prove that $TOT(G)$ can not be lower than all elements of $G.$		
	\end{enumerate}
\end{proof}

 One example of proper  aggregation operation is $\mu(G),$ median on $G \subset \bb{R}.$ 
 \begin{statement}
 	Operation $\mu(G)$ is a proper aggregation.
 \end{statement}
\begin{proof}
Let us prove monotony. Denote $q$ isomorphism $G_1 \rightarrow G_2: q(x) \ge x$ and $${\rho_i  = \mu(G_i), \; i = 1,2.}$$  Because the sets are isomorphic, they have the came power ${\|G_1\| = \|G_2\| = n.}$  

For some integer $k: \; n = 2 \, k$ or $n = 2\, k + 1$. 
The number of elements in $G_2$ which are larger than $\rho_2$ is the same as the number of elements larger than $\rho_1$ in $G_1$. In both cases and for both sets the number is equal $k$. 

. 
Denote 
$$G_i^- = \{x \; | \;(x \in G_i) \; \& \; (x < \rho_1) \},  \text{  for  } i = 1, 2.$$
$$G_i^+ = \{x\;  | \;(x \in G_i) \; \& \; (x > \rho_1) \},  \text{  for  } i = 1, 2.$$

By definition of $q$,  for any $x \in G_1^+,\; q(x) \in G_2^+. $
So, $\|G_2^+\| \ge \|G_1^+\|.$

Suppose,  $\|G_2^+\| >  \|G_1^+\|.$  
This means, $\|G_2^+\| > k.$  Therefore $\rho_2 \in G_2^+$ and $\rho_2 > \rho_1$. It proves the theorem for the case $q(x) \ge x.$ 

Suppose, $\|G_2^+\| =  \|G_1^+\| = k$  
In this case, $\|G_2^-\| =  \|G_1^-\| = k$  and for every element $x \in G_1^-$ $q(x) \in G_2^-.$

First, suppose $n = 2 \, k + 1.$ 
Then $\rho_1  \in G_1.$ The only element of $G_2$ which does not belong to $G_2^-, G_2^+$ is $q(\rho_1)$. Therefore $q(\rho_1) = \rho_2$, and $\rho_2 \ge \rho_1. $

Now, suppose $n = 2 \, k$ In this case, for $i = 1,2$
$$\rho_i =  \mu(G_i) = \frac{\min(G_i^+) + \max(G_i^-)}{2}.$$

Since $q(\max(G_1^-) ) \ge \max(G_1^-)$ and $q(\max(G_1^-) ) \in G_2^-$ then $\max(G_2^-) \ge \max(G_1^-).$

Let us notice that  $x \in G_1^+$ if and only if $q(x) \in G_2^+$. For any $x \in G_1^+:\; q(x) \ge x \ge \min(G_1^+)$ therefore $\min(G_2^+) \ge \min(G_1^+).$ It follows that $\rho_1 \le \rho_2$ in this too. 

This proves the monotony for the case, when $q(x) \ge x.$
The case when for every $x: q(x) > x$ is proven similarly. 

Let us prove the idempotence. Denote $\rho = \mu(G), \; b < \rho < c $ are two closest elements in $G$ to $\rho$. Suppose, $\|G\| = 2k$ and $\rho = (b + c)/2.$  Then   $\mu(G \cup \{\rho\} ) = \rho. $ Suppose $\|G\| = 2k + 1.$ Then $G \cup \{\rho\}$ has has two identical elements equal $\rho$ in the middle. And $\mu(G \cup \{\rho\}) = \rho.$

Tautology is trivial, because median of $\{x\}$ is $x$. 
\end{proof}

The statement  could be proven not only for median, but for any percentile. So any percentile can be used as a proper aggregation.

\subsection{Total proper \gry}

If an \gry theory has $k$  aspects, then each hypotheses will be characterized by $k$ sets of deviations. To compare hypotheses,  the deviations need to be aggregated. For this purpose, we use two step procedure: first deviations for each aspect are aggregated using it own  proper aggregation, then the results of these operations are further aggregated (along with some regularization components, possibly) to have a single number representing \gry  for a given hypothesis.  

The result of the aggregation of deviations for a single aspect will be called \textbf{aspect \gryp}
Given training set $S$, hypothesis $h$ and \gry theory $T$, a procedure $TOT(G)$ will be called\textbf{ total proper aggregation} procedure 
if it satisfies three conditions: 
\begin{enumerate}
	\item The set $G$ contains all aspect \grym each obtained with a proper aggregation  on the  full model of $S, h, T.$
	\item In addition, the set $G$ may contain regularization components.
	\item The operator $TOT(G)$  is \textbf{isotone}:  For any multisets $G_1, G_2$ and real numbers $x, y$ $$( G_1 =  (G \setminus \{x\}) \cup \{y\}) \; \& \; (y \ge x) \Rightarrow TOT(G_1) \ge TOT(G).$$	
\end{enumerate} 
The result of applying a  total proper aggregation procedure on aspects of \gry and regularization components will be called \textbf{total proper \gryp}

\subsection{Logic of recursive aggregation}

One drawback of using percentiles for aggregation is, perhaps,  the non-linear calculation complexity. Learners usually prefer to use aggregation which requires   going through all the elements of  the multiset $G$ only  once.

To express recursive aggregation, I will use extension of the first order logic with added counting quantifiers \cite{Libkin} $\exists^{=c} x, $  where $x$ is a variable, and $c$ may be a natural number or variable with values in $\bb{N}. $ The quantifier means: 
there exists exactly $c$ of $x.$

There are three sorts. 

\begin{table}[h!]
	\begin{center}
		\caption{Sorts}
		\begin{tabular}{| l | l | l | l| }
			\hline
			& Sort  & Meaning  & Variables  \\ \hline
			1 &  $\bb{G}$ & finite set of real numbers& $ x, y, z  $ \\
		  2 & $N$ & $ N = \{1, \ldots, n\}, \; n  = \|\bb{G}\|$   & $i, n, i_1, n_1, \ldots$ \\
			3 & $\bb{R}$ &  space of real numbers & $r, r_1, r_2, \ldots$  \\					
			\hline
		\end{tabular}
	\end{center}
\end{table}

There is total order $ \prec $ on the domain $G.$ The functions in the language of aggregation  are described in the next table

\begin{table}[h!]
	\begin{center}
		\caption{Function symbols}
		
		\begin{tabular}{|l | l | l | l | l |  }
			\hline
			& Symbol   & Arity & Sorts  of arguments  & Semantic   \\ \hline
			1 & $\bm{scale} $  & 1   & $\bb{G}  \rightarrow \bb{R} $ & scaling    \\
			2 & $\bm{count}$  & 0 & $\emptyset  \rightarrow N$ & cardinality of $\bb{G}$ domain    \\
			3 &$ \bm{plus}$ & 3  & $\bb{R}, \bb{R} \rightarrow \bb{R}$ &  compounding \\
			4 & $\bm{agg}$  & 1  &  $ N\rightarrow \bb{R}$  &  recursive aggregation  \\
			5 & $\bm{get}$  & 1 & $N\rightarrow \bb{G}$ & get i-th in order $<$   \\
		    6 & $ \bm{norm}$ & 2  & $ \bb{R}, N  \rightarrow \bb{R}$ &  normalization   \\
			\hline			
		\end{tabular} 
	\end{center}
\end{table}

\newpage
\subsubsection{Theory of recursive aggregation}

The order $\prec$ on $\bb{G}$ is defined as a strict total  order (with axioms of irreflexivity, transitivity, anti- symmetry and total order). The relationships  $\{<, >, \le, \ge, =\}$ are defined in usual way on real numbers. 
The functions $\bm{get}(i), \bm{count}()$ are defined uniquely as $i$-th element in the order $\prec$ and the cardinality of $G$ when the domain $G$ of the sort $\bb{G}$ and the order $\prec$ on it are known:
$$\forall x  \forall i  \; (\bm{ge}t(i) = x ) \Leftrightarrow (\exists^{=i-1}y  \;\; y \prec x )  $$
$$ \forall n \; \bm{count}() = n \; \Leftrightarrow \; \bm{count}().$$
The  table shows axioms characterizing   properties of other functions in the language:  

\begin{table}[h!]
	
	\caption{Axioms of recursive aggregation } 	\label{Axioms2}
	\begin{tabular}{l  l  r c|  l }
		\hline
		& & Axiom & & Commentary\\ \hline
		$\mathcal{B}_{1} $ & $\forall x \forall x_1 $ &
		 $(x_1 > x) \Rightarrow  (\bm{scale}(x_1) \geq \bm{scale}(x)) $ & & monotony \\ 
		\hline
			$\mathcal{B}_{2} $ & $\forall x \forall y $ & $\bm{plus}(x, y) = \bm{plus}(y,x)$ &
		& symmetry \\ \hline
		
			$\mathcal{B}_{3} $ & $\forall x \forall y  \forall x_1 $ 
		& $ ( x_1  > x ) \Rightarrow (\bm{plus}(x_1, y)   \geq  \bm{plus}(x, y) ) $& & monotony  \\ 
		\hline
		
			$\mathcal{B}_{4} $ & $\forall x \forall y  \forall z $ 
		&$ \bm{plus}(x, \bm{plus}(y, z) )  = \bm{plus}( \bm{plus}(x, y), z) $ & & associativity  \\ 
		\hline
		
	   $\mathcal{B}_{5} $ & $\forall i$& $\bm{agg}(1) = \bm{scale}(\bm{get}(1)) $ & $\&$ & recursive\\ 
	    & & $ ( \bm{agg}(i + 1) = \bm{plus}( \bm{agg}(i) , \bm{scale}(\bm{get}(i + 1 ))) $ &  & aggregation \\	
	    \hline
	    
	    $\mathcal{B}_{6} $ & $ \forall x \forall n \forall x_1 \forall n_1$ &
	    	            $ (x_1 > x) \Rightarrow ( \bm{norm}(x_1, n) \geq \bm{norm}(x, n) ) $ & \& & $\bm{norm} $\\
	    	           && $(n_1 > n) \Rightarrow ( \bm{norm}(x , n_1) \leq   \bm{norm}(x , n)) $ & & monotony \\
	    	           \hline
	   $\mathcal{B}_7$ & $\forall i$ & $\Big(\bm{get}(i + 1) = \bm{norm}( \bm{agg}(i), i)\Big) $&$ \Rightarrow $& idempotence \\
	   & & $\bm{norm}( \bm{agg}(i+1), i + 1) = \bm{get}(i+1)$ & &\\ \hline
	    $\mathcal{B}_{8}$ & $\forall x$ & $\bm{norm}(\bm{scale}(x), 1) = x $ && tautology \\ \hline
	    
	    $\mathcal{B}_{9}$ & $\forall x_1 \forall x_2 \forall i$ &
	    $(x_1 \neq x_2) \Rightarrow  (\bm{scale}(x_1) \neq (\bm{scale}(x_2)  $ & $\&$ & strict \\ && $\bm{norm}(x_1, i) \neq \bm{norm}(x_2, i)  $ & & monotony  \\\hline 
	    
	   $\mathcal{B}_{10} $& $ \forall x_1 \forall x_2 \forall y_1 \forall y_2$ &
	   $(x_1 < x_2) \& (y_1 < y_2) $ & $\Rightarrow$ & strict \\
	   && $\bm{plus}(x_1, y_1) < \bm{plus}(x_2, y_2)$ && monotony \\ \hline
	 
	\end{tabular}
\end{table} 

Typical examples of the function $plus$ are
\begin{itemize}
	\item $plus(x, y) = x + y$
	\item $plus(x, y) = x \cdot y$
	\item $plus(x, y) = max(x, y).$
\end{itemize}
All these functions are used by popular learners, as I will demonstrate. 

Each model $\mathcal{M}$ uniquely defines an operation   $${TOT}(\mathcal{M}) = \bm{norm}( \bm{agg}(\bm{count} ), \; \bm{count} ) $$ 

Given an interpretation of functions $\bm{scale},$ $\bm{plus}$ $\bm{agg},$ $   \bm{norm}$,   ${TOT}(\mathcal{M}) $ is defined  by its  finite domain $\bb{G}$ and the strict total order $\prec$ on it.  

The next theorem shows that  ${TOT}(\mathcal{M})$  does not depend on the order $\prec$.

	 \begin{theorem}
		Suppose  $\mathcal{M}_0, \mathcal{M}_1$ models of recursive aggregation language are different by the orders $\prec$ only:  domains of the sort $\bm{G}$ consist of the same elements and interpretations of all the functions of the language are identical.   Then $TOT(\mathcal{M}_0) = TOT(\mathcal{M}_1).$
	 \end{theorem}
 \begin{proof}
 	Suppose, the models are different by the orders on domains $G_0, G_1$ of sort $\bm{G}.$
 	For the finite domain $G, $ the  order $\prec_1$  may be considered as a permutation of  order $\prec$ 
 	Each permutation can be obtained  by finite number of simple transpositions (transpositions of  neighboring elements). Suppose, the order $\prec_1$ is obtained from order $\prec_0$ by $K$ simple transpositions. Let us prove the theorem with  induction by $K$. First, suppose $K = 1.$
 	Denote $x_1, \ldots, x_i, x_{i+1}, \ldots, x_n$ elements of $G$ ordered by $\prec_0$. Suppose, the order $\prec_1$ transposes elements $x_i, x_{i+1}$. Denote $\bm{agg}(l), \bm{agg}_1(l)$ values of the recursive aggregation function  obtained on the step $l$ with the orders $\prec_0,  \prec_1$ respectively. Since all the elements prior to $i$ are identical in these orders, $\bm{agg}(i-1) = \bm{agg}_1(i-1).$ By definition 
 	\begin{align*} 
 	   \bm{agg}(i) & =  \bm{plus}( \; \bm{agg}(i-1), x_i \;) &   \\
 	   \bm{agg}(i+1) & = \bm{plus}( \; \bm{agg}(i), x_{i+1} \; )  \\
 	   & = \bm{plus}( \; \bm{plus}( \; \bm{agg}(i-1), \; x_i\; ), \; x_{i+1} \;)\\
 	   \bm{agg}_1(i) &=  \bm{plus}( \; \bm{agg}(i-1), \; x_{i+1}\;)  & \\
 	    \bm{agg}_1(i + 1) & = \bm{plus}( \; \bm{plus}( \; \bm{agg}(i-1), \; x_{i+1}\;), \; x_i \; ).\\
 	\end{align*}
Using symmetry and associativity of the function $\bm{plus}$ ($\mathcal{B}_{2}, \mathcal{B}_{4}$) we get
\begin{align*}
	 \bm{agg}_1(i + 1) & = \bm{plus}( \; \bm{plus}( \; \bm{agg}(i-1), \; x_{i+1}\;), \; x_i \; )\\
	                                & =  \bm{plus}( \;  x_i,  \; \bm{plus}( \; \bm{agg}(i-1), \; x_{i+1}\;) ) \\
                               & = \bm{plus}( \; \bm{plus}( x_i, \; \bm{agg}(i-1)) , \; x_{i+1}\;) \\
                                &  = \bm{plus}( \; \bm{plus}( \; \bm{agg}(i-1), \;  x_i \; ) , \; x_{i+1}\;) \\
	                                & = \bm{plus}( \; \bm{agg}(i), \;x_{i+1} \;)\\
	                                & = \bm{agg}(i+1). \\	
\end{align*}

All the elements in the orders $\prec_0, \prec_1$ after $(i+1)$-th are identical. Therefore, $\bm{agg}_1(n) = \bm{agg}(n).$ So, the constants $out(\mathcal{M}) $ and $ out(\mathcal{M}_1)$ will be identical in this case. 

Suppose, we proved the theorem for $K = k$. Let us prove it for $k+1.$ Suppose, the first $k$ simple transpositions involve elements with the indices below $i-1$, and the last simple transposition involves elements $x_i, x_{i+1}.$ Then, the same considerations apply again. 
 
 \end{proof}

Every interpretation of the functions $\{\bm{scale},$ $\bm{plus}$ $\bm{agg},$ $   \bm{norm}\}$  will have potentially infinite number of models different by the domains of the sort $\bb{G}$

The theorem means that, given interpretation of the  functions of the recursive aggregation language, the operation $TOT(\mathcal{M})$ is an aggregation operation on the domain $G$ of sort $\bb{G}:  TOT(\mathcal{M}) =  TOT(G) $. 

An aggregation defined by an interpretation of the recursive aggregation language may be called recursive aggregation. 

The next theorem shows recursive aggregation  is a proper aggregation. 

\begin{theorem}
	For any interpretation of the recursive aggregation language, the operation $TOT(G)$ is a proper aggregation. 
\end{theorem}
\begin{proof}
	Let us prove tautology. If $G = \{x\}$, then $$TOT(G) = \bm{norm}(\bm{agg}(1), 1) = \bm{norm}(s\bm{cale}(\bm{get}(1)) ,1) = \bm{norm}(\bm{scale}(x) ,1) = x,$$ using the axiom $\mathcal{B}_{8}.$ 
	
	 Let us prove monotony. Suppose, two models with domains $G_1, G_2$ of the sort $\bb{G}$ belong to the same interpretation, and $q: G_1 \rightarrow G_2$ is isomorphism such that $q(x) \ge x$. Suppose they are ordered in such a way that $q$ maps $i$-th element of $G_1$ into $i$-th element of $G_2$. 
	 
	 Let us prove it by induction by $n = \|G_1 \| = \|G_2\|.$ For $n = 1$ it is true based on the axiom $\mathcal{B}_8.$ Suppose, the statement is proven for $n = k.$ 
	 Denote $agg_1(i), agg_2(i)$ results of aggregation on the domains $G_1, G_2$ on the step $i$, and denote corresponding elements of $G_1, G_2: \; x_j, \; y_j = q(x_j), \; j = 1, \ldots, n.$ 
	 By the assumption of induction,  
	 $\bm{norm}(\bm{scale}bm{agg}_2(k)), k) \ge \bm{norm}(\bm{scale}(\bm{agg}_1(k)), k)$ 
	 
	  Let us prove the statement  for $n = k +1.$ 
	  $$\bm{norm}(\bm{scale}(\bm{agg}_1(k+1)), k+1) = 
	\bm{norm}(\bm{scale}(\bm{plus}(agg_2(k), \bm{scale}(x_{k+1}), k+1) .$$
	  $$\bm{norm}(\bm{scale}(\bm{agg}_2(k+1)), k+1) = 
	\bm{norm}(\bm{scale}(\bm{plus}(agg_1(k), \bm{scale}(y_{k+1}), k+1) .$$
	We need to show that
	$$\bm{norm}(\bm{scale}(\bm{plus}(agg_2(k), \bm{scale}(y_{k+1}), k+1)  \ge $$
$$	\bm{norm}(\bm{scale}(\bm{plus}(agg_1(k), \bm{scale}(x_{k+1}), k+1) .$$ 
The function $\bm{norm}$ is isotone by the first argument ($\mathcal{B}_6)$. The function $\bm{scale}$ is isotone ($\mathcal{B}_1$). The function $\bm{plus}$ is isotone by the first argument ($\mathcal{B}_3$) and symmetric ($\mathcal{B}_2$), therefore it is isotone by both arguments. It follows that inequality holds.
The strict monotony follows from monotony and the axioms $\mathcal{B}_{9}, \mathcal{B} \mathcal{B}_10.$

	Denote $G_0,\;  G_1 = G_0 \cup {r}$. To prove idempotence, assume $r = TOT(G_0).$ 
	For $G_1$, $\bm{get}(n+1) = r = \bm{norm}(\bm{agg}(n), n).$
	Using axiom $\mathcal{B}_7$ we get 
	$$TOT(G_1) = \bm{norm}(\bm{agg}(i+1), i+1) = r = TOT(G_0).$$

\end{proof}

It is easy to show that each of the next  combinations of functions satisfies all the axioms $\mathcal{B}_1 - \mathcal{B}_{10}.$
\renewcommand{\arraystretch}{2}
\begin{table}[h!]
	
	\caption{Some interpretations of recursive aggregation \label{tab:aggregations}} 
	\begin{center}
	\begin{tabular}{| c |c |c| c |c| }
		\hline
& $\bm{plus}(x, y) = $ & $\bm{scale}(x) = $ & $\bm{norm}(x, i) =$  & $TOT(G) =$\\ \hline
  1 &$ x  + y$ & $x$ & $x / i $ & $ \frac{1}{n} \sum_i x_i $ \\ \hline
  2 & $ x + y$ & $ x^2 $ & $\sqrt{ x/ i } $ &$ \sqrt{\frac{1}{n}{\sum_i x_i^2} }$\\ \hline
  3 & $ max(x, y) $ & $x$ & $x$ & $max(G)$\\ \hline
 4 & $ x \cdot y $ & $x$ & $ x^{ 1/i} $ & $ ( \prod_i x_i )^ {1/n}  $\\ \hline
\end{tabular}
\bigskip
\\ $G$ is domain of the sort $\bb{G}$, $n = \|G\|.$
\end{center} 
\end{table} 
\renewcommand{\arraystretch}{1}

\begin{statement} Each  combination of functions in the table \ref{tab:aggregations}  is an  interpretations of the recursive aggregation language. 
	\end{statement} 
	\begin{proof}
		Monotonicity of all the functions is obvious. Symmetry of the function $\bm{plus}$ in all the combinations is obvious. We need to show that axioms of tautology and idempotence are true for all combinations.
		Let us show it for combinations from the lines (2) and (4), where it may be not obvious.
	 	If $G = \{x\}$, $TOT(G) = \bm{norm}( \bm{agg}(1), 1) = \bm{norm}( \bm{scale}(x), 1).$	 	 For  the combination (2) $TOT(G) = \sqrt{x^2/ 1} = x.$
	 	For the combination (4) $TOT(G) = (\prod x)^1 = x. $
	 	Let us prove idempotence. For combinations (1) and (3) it is obvious. Let us prove it for combinations (2) and (4) again.
	 	Suppose, $z = x_{n+1} = TOT(\{x_1, \ldots, x_n\}).$
	 	
	 	For combination (2)
	 	\begin{align*}
	 			z & = TOT(\{x_1, \ldots, x_{n}\}) & =  \sqrt{ \frac{ \sum_{i = 1}^n x_i^2  } {n}} \\
	 			& TOT(\{x_1, \ldots, x_{n}, z\} ) &  = \sqrt{ \frac{ \sum_{i = 1}^n x_i^2  + z^2 }{n+1}}\\
	 				& = \sqrt{  \frac{ \sum_{i = 1}^n x_i^2  + { ( \sum_{i = 1}^n x_i^2 )  } / {n} }{n+1}}
	 				& = \sqrt{ \frac{ \sum_{i = 1}^n \frac{n+1}{n} x_i^2}{n+1} }\\
	 				 & = \sqrt{ \frac{\sum_{i = 1}^n x_i^2 }{n} }
	 				& = TOT(\{x_1, \ldots, x_n\}\\
	 	\end{align*} 
 	   For combination (4)
 	   $$z = TOT(\{x_1, \ldots, x_n\}) = (\prod_i x_i)^{1/n}$$
 	   \begin{align*}
 	   TOT(\{x_1, \ldots, x_n, z\} ) & = \left(  \left(\prod x_i \right) \cdot  z \right) ^{1 / {n+1}}\\
 	   & =  \left( \prod_i x_i \cdot  \left(  \prod_i x_i \right)^{1/n}\right) ^{1/{n+1}} \\
 	   & = \left( \prod_i x_i^{(n+1)/n} \right)^{1/{n+1}} \\
 	   & =  \prod_i x_i^{1/n} = \left( \prod_i  x_i \right)^{1/n} = z.\\
	 \end{align*} 
	\end{proof}

In the next section, I show examples of how to evaluate \gry in some real life situations.

\section{Incongruity evaluation for data analysis} \label{Semantics}

The concept of \gry was developed for evaluation of predictability, to give logical foundation for learning in ML. The next examples show that the concept may be used for wide array of data analysis problems, where we evaluate various assumptions about dependencies of interest. 

\subsection{Incongruity  of scales}

 Many people have body weight  scale. I have two.  The assumption is that they measure the same weight with small errors. Disagreements  between the imprecise scales may be formalized as \gry to evaluate validity of this assumption. As in the case of ML, the estimate of \gry may be used for decision making: shall I go on with these scales, or buy a new, more reliable one. 

In this case, the underlying dependence $\ph$  is the dependence of my ``true'' weight on time. The Language of observations and hypotheses (LOH) has two modalities $\approx_1, \approx_2$ corresponding to readings from the first and the second scale respectively.   So, all the formulas of the observations have the form
$$\approx_i (\ph(x) = y),$$
where $i= 1, 2$ is the index of the scale,  $x$ is time, $y$ is  weight. 

 The collision condition shall identify pairs of measurements of two scales taken within small interval of time. For these pairs larger differences of weight mean larger errors.
The collision condition is defined by the formula:
$$\pi(\alpha_1, \alpha_2) = \big((\bm{s}(\alpha_1) = \approx_1 )\; \& \;( \bm{s}(\alpha_2) = \approx_2) \;  \& \; (\rho_x(\alpha_1, \alpha_2) \le 5min)\big).$$
with the deviation function   
\[
t(\alpha_1, \alpha_2) = 
\begin{cases}
	0, & \text{if  } \rho_y(\alpha_1, \alpha_2) < 1\\
\rho_y(\alpha_1, \alpha_2) - 1, & \text{otherwise}. 	
\end{cases}
\]

For the proper aggregation of the deviations I use 80-th percentile. If $20\%$ of deviations  are positive, the scales can not be used. I may also use maximum. If the maximal deviation is more than 1 pound, the scales are useless.

\subsection{Is there a dependence? } 

 Suppose, I want to check an assumption that my weight is a non-decreasing  function of the amount of  consumed  calories.   The goal is to discover the actionable pattern.  Again, we can use \gry to make a conclusion. 
 
 Let us assume that when the amount of calories from day to day changes less than 100 calories, it may not affect on my weight;  and the weight is  evaluated with precision 1 pound. 
 
 In this case, the underlying dependence $\ph(x)$ is a dependence of the recorded weight on the consumed calories. 

There is only one modality $\approx$ so all the formulas in the language of observations and hypotheses 
have the form $\approx (\ph(x) = y).$ 

Since the assumption is that  my weight is a monotone function of consumed calories,  the next situations violate the assumptions
\begin{itemize} 
	\item  when the weight  changes in one direction, but consumed calories change in opposite direction; 
	\item when I consume about the same amount of calories,  the but the weight changes. 
\end{itemize} 
In both cases, the larger are the differences in weight (variable $y$), the larger shall be deviations. 

Accordingly, there shall be two aspects of \gry with these collision conditions: 
\begin{align*}
\pi_1(\alpha_1, \alpha_2) = \Big( ( \bm{x}(\alpha_1) < \bm{x}(\alpha_2)) \; \& \;(  \bm{y}(\alpha_1) >  \bm{y}(\alpha_2))  \; \& \;
( \rho_x(\alpha_1, \alpha_2) > 100) \Big) 
\end{align*} 

$$ \pi_2(\alpha_1, \alpha_2) =  \big( (\bm{y}(\alpha_1) > \bm{y}(\alpha_2) ) \; \& \; (\rho_x(\alpha_1, \alpha_2)  < 100) \big) .$$ 

In both cases the  deviation function is 
\[
t(\alpha_1, \alpha_2) = 
\begin{cases}
	0, & \text{if  } \rho_y(\alpha_1, \alpha_2) < 1\\
	\rho_y(\alpha_1, \alpha_2) - 1, & \text{otherwise}. 	
\end{cases}
\]

I would use the first   combinations of functions from the table \ref{tab:aggregations} of typical interpretations of the Language of recursive aggregation to get a handle on how big are the deviations, and if the dependence is strong enough. 

I may study various independent variables and their combinations to see if some of them are associated with the weight better. If I cannot assume monotonicity of the dependence, only the second aspect of \gry will be used.  The lower is the total \gry of a hypothesis, the more likely there is  the dependence which can be used to control  weight. 

Usually, the correlation is evaluated using coefficient of correlation in statistics. Statistics answer the next question: how likely is that the correlation exists in general population? 
First of all, the concept of ``general population'' does not make sense for  my weight measurements. Then, my question is not about  existence of dependence, but if there is strong enough  dependence  to use  for  prediction of my weight changes. This is completely different question. The values of the regression coefficient depend on the sample size, which is extremely important for the statistical question, and irrelevant for prediction.

%
\bigskip
\subsection{Tracking Oswald}

Many witnesses reported seeing Lee Harvey Oswald during  the day of Kennedy assassination. The investigators may have several theories about his  movements  on this day.  Incongruity evaluation may be used to find the theory  maximally reconciled with witnesses testimonies, even though some of them may be unreliable.

The  underlying dependence  $\varphi$ reflects the true movements of Oswald. It maps times (variable $x$)  into locations (variable $y$)  with particular memorable names (such as ``the sixth floor of the Texas School Book Depository''). 

The distance  in   time (by variable $x$)  is measured in minutes.  The distance between  locations is also measured in minutes: the time required to get from one place to another.  Formulas (instances) of observations  describe   locations and times of Oswald's sightings by witnesses and have form $\approx (\ph(x) = y).$ Formulas of hypotheses (hypothetical instances) reflect the investigator's theory, and have the form $\asymp (\ph(x) = y).$

So, for two instances  (formulas of LOH) $\alpha_1, \alpha_2,$ $\rho_x(\alpha_1, \alpha_2)$ is the time which elapsed between the (hypothetical, observed) sightings, and $\rho_y(\alpha_1, \alpha_2)$ is the minimal time, required to travel between the locations $\bm{y}(\alpha_1), \bm{y}(\alpha_2)$. 
The times between sightings in any two locations shall not be less than the minimal time required to travel between these locations: $\rho_y(\alpha_1, \alpha_2) \leq \rho_x(\alpha_1, \alpha_2).$

For example, if Oswald was seen in the location $A$ in the time $t_1$ (observation $\alpha_1$), and hypothetically he was in some location $B$ in time $t_2$ (hypothetical instance $\alpha_2$) and the time to travel between  $A$ and $B$ is $z$, then for the observation and the hypothesis to be both  true it is required that $|t_1 - t_2| > z. $

 Accordingly, the \gry theory has only aspect with the collision condition: 
 $$\pi(\alpha_1, \alpha_2) = \Big( ( \bm{s}(\alpha_1) = \asymp) \;  \&  \; (\bm{s}(\alpha_2) = \approx) \Big),$$
 which simply identifies the formula $\alpha_1$  as a hypothetical instance, and the formula $\alpha_2$ as an observations.

The deviations are calculated by formula
\[
\delta(\alpha_1, \alpha_2) = 
\begin{cases}
	0, & \text{if } \rho_y(\alpha_1, \alpha_2) \le \rho_y(\alpha_1, \alpha_2)\\
	\rho_y(\alpha_1, \alpha_2) - \rho_x(\alpha_1, \alpha_2), & \text{otherwise}.
\end{cases}
\]

For a given hypothesis (investigators theory) the full model will include all the witnesses observations and hypothetical  formulas with times  of Oswald's whereabouts in all the locations mentioned by the witnesses.   

For proper aggregation of deviations, I would use the first  combination of functions from the table \ref{tab:aggregations} of typical interpretations of the language of recursive aggregation. 

The theory with the lowest  \gry  may be considered optimal.
The advantage of this approach is that there is no subjectivity in evaluation of witnesses testimonies and  theories of Oswald's movements. 

The same way, as witnesses testimonies are evaluated for incongruity with the theories, the testimonial of one witness can be compared with  testimonies of all others. Incongruity of each witness with other witnesses can be used to, may be, exclude exceptionally contradictory witnesses.

\section{Structure of learners} 

Now I concentrate on the  learners.

\subsection{The Main Conjecture} 

In practical applications,  the training set is the set of given observations $\{ \beta_i, i = 1:m\}.$  

The procedures of $k$-NN, Naive Bayes, SVM, hierarchical clustering,  for example, appear to have nothing in common - because  they  are formulated in different terms.

Here I propose a common language to describe  procedures used by learners.

\begin{tcolorbox}[title= The maint conjecture]
\begin{enumerate}
\item  Let $F$ be a class of hypothesis for a learner $\mathcal{Z}$. There exists   a loss criterion $L_{\mathcal{Z}}(h, S), \; h \in F, $ such that, given a training set $S$ and parameters $q$,  
the   learner $\mathcal{Z}$   performs

\textbf{Proper training} 
\begin{itemize}
	\item\textbf{Optional, Focusing}: transformation   $U: \; S \rightarrow  S_q$ 
	\item \textbf{Fitting}: generation of the  hypotheses $h \in F$   and evaluation of 
	$L_{\mathcal{Z}}( h, S_q) $ 
	\item  \textbf{Optimal selection}:  selecting  a hypothesis $h^\prime(q)$  with lowest  loss $L_{\mathcal{Z}}(h, S_q)$  as a decision.
\end{itemize} 
If $\mathcal{Z}$ is a wrapper-type learner, it  has the next steps performed in a loop by $i$:
\begin{itemize}
	\item \textbf{Generating  parameters} $q_i$ 
	\item \textbf{Proper training} with parameters $q_i$
	\item  \textbf{Calculating weight } $W( q_i)$ 
	\item   \textbf{Combining decisions: } ${d =  \Psi\Big(\{h^\prime(q_i),  W(q_i)\} \Big).}$
	\end{itemize}
\item There exists a \gry  theory $T$,  regularization functional $R(h)$  and a total proper aggregation procedure $\tau$ such that for a hypothesis $h$ and observations $S$
$$ L_{\mathcal{Z}}(h, S) = \tau(h, S, T, R(h)), $$
total proper incongruity of $h, S, T.$
\end{enumerate} 
\end{tcolorbox}

 On a Focusing step  transformation $U$ may be non-linear transformation of data prior to building a model.

Yet, typically, focusing is used to select observations or features or emphasize some of them  with weights. 

The ``lowest loss'' is, usually, a minimal loss among the tested hypotheses. It may coincide with the lowest loss on $F$ or not. 

The procedures may use sequence control operators: loops, breaks and so on. 

The main conjecture  answers  the first fundamental question:  \textbf{What shall be done with the training set for learning? }The main conjecture defines the steps used by each learner,  describes a loss criterion as \gry.

\section{Popular learners support the  Main Conjecture} 

\subsection{ERM-type learners} 

Denote $\beta_i = \big( \approx (\ph(x_i) = y_i) \big) $ i-th observation  in the training set $S$. 

In this case, class of functions is not specified and the procedure is not described. All we have is a loss criterion

$$L(h, S) = \frac{1}{m} \sum_i |h( \bm{x}(\beta_i)) - \bm{y}(\beta_i)   | .$$

Let us demonstrate that the loss criterion is the incongruity of the hypothesis $h$, training set and the \textbf{Point-Wise predictability theory}, $T_{pw}$ (see ( \ref{pointwise})).

For any hypothesis $h$, denote  $\Upsilon$ all the formulas of  the full model $\mathcal{M}(h, S, T_{pw}).$  By the definition of the full model for the theory $T_{pw}$,  $ S \subseteq \Upsilon, $  and $\Upsilon$  contains  the hypothetical instances of the hypothesis $h$ defined in the same data  points $x_i$  as observations. 

So, for each pair of formulas $\alpha_1, \alpha_2$ from $\Upsilon$ satisfying the collision condition $\pi_{pw}$, the deviation 
is $$\delta(\alpha_i, \alpha_2) = |\bm{y}(\alpha_1) - \bm{y}(\alpha_2)| = | h( \bm{x}(\alpha_2)) - \bm{y}(\alpha_2)|,$$
where $\alpha_2 \in S.$

Total proper aggregation here is proper aggregation for the only aspect of  \gryp If we use the proper aggregation operation  defined in the first line of the table \ref{tab:aggregations},  then

$L(h, S) = \sum_{\beta \in S} | h(\bm{x}( \beta) - \bm{y}(\beta) |= \gamma(h, S, T_{pw}).$

Thus, the loss criterion empiric risk is a total proper \gry, and it supports the main conjecture. 

\subsection{Linkage-based clustering}

The learner is also popularly known as hierarchical clustering. 

Intuitively, clustering is a learning problem, because it is about  modeling of a predictable  dependence: close data points shall belong to the same cluster. 

 In \cite{Shalev}, a general concept of linkage-based clustering is introduced this way:
\begin{quotation}
	These algorithms proceed in a sequence of rounds. They start from trivial clustering that has each data point in a single-point cluster. Then, repeatedly, these algorithms merge ``closest'' clusters of the previous clustering. $\langle \ldots \rangle$ Input to a clustering algorithm is between-point distance, $d.$ There are many ways of extending $d$ to a measure of distance  between domain subsets (or clusters. The most common ways are.
	\begin{enumerate}
		\item Single Linkage clustering, in which the between-clusters distance is defined by the minimum distance between members of the two clusters $\langle \ldots \rangle$
		\item Average Linkage clustering, in which the distance between two clusters is defined to be average distance between a point in one of the clusters and a point in another $ \langle \ldots \rangle$
		\item Max Linkage clustering, in which the distance between two clusters is defined as maximum distance between their elements $ \langle  \ldots\rangle $
		\end{enumerate} 
\end{quotation}

The last option clearly contradicts declared goal ``merge `closest'  clusters''. But I will consider it too.

Close observations shall belong to the same cluster. The opposite is also true:
observations of the same cluster shall be some-what close to each other. The last dependence is used for clustering. So, we consider cluster number as an independent variable $x$, and the observation vector as dependent variable, $y$.

The training set is a sequence of formulas 
$$\{ \approx (\ph( c_i) = y_i), i = 1:m \}, $$
where $c_i$ is a cluster if $i$-th observation, and $y_i$ is the observed vector of the same observation.  

Denote 
$$C_i = \{y \; | \;   \exists \alpha  (\alpha \in S)  \; \& \;  (y = \bm{y} (\alpha) ) \; \& \;  (i = \bm{x}(\alpha))  \}.$$ the set of elements of the $i$-th cluster. 

The clustering consists of repeated rounds: two ``closest'' clusters are found, combined, and the procedure repeats until there is only one cluster left.  Proper learning  happens when we search for the ``closest'' clusters. On this step,   for the each cluster $i,$  we check each of the clusters $j > i$  and evaluate their ``distance''; then the two ``closest'' clusters are selected for combining.

When we evaluate the ``distance'' between clusters $A, B,$ it is convenient to see one cluster (say, $A)$ as a target, and another $(B)$ as a candidate to combine with the target. In other words, we evaluate  the hypothesis that elements of the cluster  $B$ are a ``good fit'' to belong to   $A.$   Thus, the notation $h^{ij}$ will indicate a hypothesis that elements of the cluster $j$ are a ``good fit'' for the cluster $i.$

Let  $H_k = \{  h^{ji}\; | \; i < j \leq k \}$ denote  the class of all the hypotheses for the case, when there are  $k$ clusters.

The instances of the hypothesis  $h^{j, i}$ have the form  
$ H^{ji} =  \{\asymp (\ph(C_i ) = y)\;  | \; y \in C_j \}.$  
For a hypothetical instance $\alpha \in H^{ji}$ and the observation $\beta \in C_i$ their distance is $d(\alpha, \beta) = \rho_y(\alpha_, \beta),$ distance between data points $\bm{y}(\alpha), \bm{y}(\beta)$  of these two formulas.

The mis-fit between clusters  $i,j$ defined in the textbook \cite{Shalev}  may be evaluated as $\gamma(G(i,j)), $ where $\gamma$ is some aggregation operation (minimum, average or maximum), and $G(i,j)$  is the set of pairwise  distances for elements of $C_i, C_j.$  The learning procedure is searching for a hypothesis $h^{i,j} \in H_k$ with the lowest  mis-fit criterion $\gamma(G(i,j)).$ 

Thus, the function $\gamma(G(i,j))$ can be considered a loss criterion of the learner. 

The learning procedure may be described this way: 
\begin{tcolorbox}[title = Hierarchical Custering]
\begin{itemize} 
	\item \textbf{Loop }by all  $i, j: \; i < j  \le k$
	\begin{itemize}
	\item \textbf{Fitting}: For the hypothesis  $h^{j, i} $ evaluate  the loss criterion $\gamma(G(i,j))$ 
	\end{itemize} 
	\item \textbf{Optimal selection}: Select a hypothesis $h^{i, j} \in H_k$ with the minimal  value of the loss criterion $\gamma(G(i, j))$ 
\end{itemize}
\end{tcolorbox}

The steps of this procedure are as  described in the main conjecture. To see that the learner agrees with the main conjecture completely, we  just need to show that for some \gry theory, the loss criterion $\gamma(G(i,j))$ is a total proper \gry for any aggregating procedure $\gamma$ mentioned in the book. 

In this case the \gry theory is the point-wise theory $T_{pw}$ again. 

For two formulas $\alpha_1, \alpha_2$ such that $\;  \vdash \pi_{pw}(\alpha_1, \alpha_2),\;$   the deviation is 
$\delta(\alpha_1, \alpha_2) = \rho_y(\alpha_1, \alpha_2).$

Every aggregation operation $\gamma$, mentioned in \cite{Shalev}, satisfies axioms of the proper aggregation. Therefore, in every case,  the loss criterion 
 $\gamma(G(i,j) ) =  \gamma(h, S, T_{pw})$ is the proper total \gryp

This  proves that linkage-based clustering agrees with the main conjecture.  It would agree with the main conjecture not only for the aggregation operations mentioned in the book, but also for any other proper aggregation operation. 

\subsection{k-NN}

 The observations  have binary feedback in $Y = \{0, 1\}$.  Given a new data point $x \in \chi$, the goal is to output prediction $f(x)$  of the underlying dependence $\varphi : \chi \rightarrow Y.$ Thus, feedback is defined in one point $x$ each time. 

The procedure can be described in these steps. 

\begin{tcolorbox}[title = $k$-NN]
	\begin{itemize}
	\item \textbf{Focusing:} selecting focus training set $Q(x)$ of $k$ observations with data points closest to $x.$

  \item  \textbf{Fitting:} evaluate error rate of each of the constant functions $0,1$ on $Q(x).$

	\item\textbf{ Optimal selection: } Selection of the constant function with minimal error rate. 
\end{itemize}
\end{tcolorbox}

The learner minimizes error rate, which is the same as 
 empiric risk $L(f, S)$ on functions $f \in \{0, 1\} $ defined on focus training sample $Q(x)$.
We have already demonstrated that empiric risk is total proper \gry for the point-wise \gry theory. 

Thus, $k$-NN supports the main conjecture. 

\subsection{Two k-NN learners with adaptive choice of $k$ } 
 
The parameter $k$ defines the size of the focus training sample. Optimally, for most of data points  $\xi \in \chi, $ the neighborhood $Q_k(\xi)$  shall be small enough to have  majority of the points of the same class as the point $\xi$ due to the ''predictability''  of the underlying dependence, and large enough of that random outliers did not confuse us. 

Here I discuss two approaches to select $k$ for every new data point. The first is described in \cite{kNN}, the second is my new algorithm. Both learners find prevalent class $y$ in the focus sample, calculate its frequency  $p_k(y)$  and the error rate ${r_k(y) = 1 - p_k(y)}$ the same as $k$-NN.

The work \cite{kNN} proposes, given a data point $x,$ start with a small $k$ and  gradually increase it  while calculating bias ${t_k(y) = p_k(y) - 0.5}$  of the prevalent class  with every $k$. The procedure stops when  the bias  reaches certain  threshold. If the threshold was not ever reached, they don't output any answer. 

The threshold  they propose to use is: 

$$ 
\Delta(n, k, \delta, c_1 ) = c_1 \sqrt {\frac{   log(n) + log( \frac{1}{\delta})  }{k }  },
$$
where $n$ is size of the training sample, $\delta$ and $c_1$ are some user-selected  parameters. 
The learner uses the same criterion as $k$-NN.

The procedure can be described like this: 

\begin{tcolorbox}[title= Ada k-NN]
	\begin{itemize}
		\item Loop by $k$  for   $k_0 \le k \le n$ 
		\begin{itemize}
		\item \textbf{Proper training:}
		\begin{itemize} 
			\item \textbf{Focusing:} Select focus training set $Q_k(x)$ of $k$ observations with data points closest to $x.$
			\item \textbf{Fitting: } Evaluate error rate $r_k(c)$  of  two constant functions $c \in \{0, 1\}$ in $Q_k(x)$ 
			\item\textbf{Optimal selection: } Select constant $c^\prime \in \{0,1\}$  with minimal error rate $r_k(c^\prime)$. 
			\end{itemize}
		  \item \textbf{Break the loop by $k$:} If $r_k(c^\prime) > \Delta(S,k,\delta, c_1)$ stop. 
		  \end{itemize} 
		\item \textbf{Combining decisions: } If $k < n$, output $c^\prime$ as decision. Otherwise, refuse to output the decision. 
	\end{itemize}
\end{tcolorbox}

Thus, the learner uses the same steps as described in the main conjecture, and it uses the same loss criterion, error rate, as original $k$-NN. Therefore, the learner agrees with the main conjecture.

This learner is developed within the statistical learning paradigm, where the training set is expected to grow to infinity fast. As $n$ increases, so does the threshold $\Delta(n, k, \delta, c_1 )$. Therefore, the selected  value $k, $  the size of the focus training set,  will go to infinity with $n.$ And thus, by the law of (very) large numbers, the solution will converge asymptotically to the expectation of the class in the given neighborhood. At the same time, the ratio of $k$ to $n$ is expected to decrease, thus the size the $k$-neighborhood will tend to 0. If the distribution is continuous in $x,$ then the leaner will likely find the solution as $n$ tends to infinity. 

The issue here is that $n$ is not going to infinity or anywhere.  For a fixed $n,$ the learner favors smaller $k$, where the evaluation of prevalent class is subject to random fluctuations caused by small sample.

To alleviate this issue, I propose an  alternative approach which uses Hoeffding inequality \cite{Shalev} to select $k$.   

The Hoeffding inequality can be written as
\begin{equation}\label{Hoef}
P[\; | \, p - E\,   | > t\;  ]  \leq 2 \;exp( - 2 k \, t^2) , 
\end{equation}
where $p$ is observed frequency of an event,  $E$ is  the expected frequency (probability) of the same event, and $t$ is an arbitrary threshold, and $k$ is the sample size. 

Suppose, $p$ evaluates observed frequency of class 1 (rate of the class 1 among the neighbors), $E$ is the probability of the class 1 in the neighborhood of a given point.  
 If $p$ is above 0.5, then observations of the class 1 prevail,  and  we pick hypothesis 1 out of two.  Otherwise, the we pick hypothesis 0. 

Let $t = | \,0.5 -p \,|. $ If $ | \, p - E\,   | > t$  the  expected prevalent class is different from the observed prevalent class.  If it is the case,  we selected the wrong hypothesis. In this case, the right side of the inequality gives us an upper limit of probability  that we picked the prevalent class wrong. 

 For selection of $k$ we use the weight, calculated as the right part of  (\ref{Hoef}) : 
$$W(y, S, k)  = 2 \cdot exp( - 2 \; k \; |\,p - 0.5\,| ^2 ). $$
Obviously, the larger is $k$, and the further is the frequency $p$ from $0.5$, the lower is the weight. The weight will serve well for the selection of the parameters $k$, because we  need to find the neighborhood where $p$ is far from uncertainty, $0.5$, yet, the size of the neighborhood is not too small. 

Here is the description of the learner's procedure for the given data point $x$.

\begin{tcolorbox}[title= Hoeffding k-NN]
	\begin{itemize}
		\item  Loop  by $k$ from $k_0$ to $n-1$ 
		\begin{itemize}
			\item \textbf{Proper training:}
			\begin{itemize} 
				\item \textbf{Focusing:} Select focus training set $Q_k(x)$ of $k$ observations with data points closest to $x.$
				\item \textbf{Fitting: } Evaluate error rate $r_k(c)$  of  hypotheses $c \in \{0, 1\}$ in $Q_k(x)$ 
				\item\textbf{Optimal selection: } Select the hypothesis $c^\prime(k)$  with minimal error rate $r_k(c^\prime(k))$. 
			\end{itemize}
			\item \textbf{Calculate wight} $W(x, S, k).$ 
		\end{itemize} 
		\item \textbf{Wrapper decision: } Select $k^\prime$ and the decision $c^\prime (k^\prime) $ with minimal weight $k^\prime = \arg \min W(x, S, k).$ 
	\end{itemize}
\end{tcolorbox}

The proper learning procedure in both $k$-NN wrappers  minimizes empiric risk, the same as original $k$-NN, and this criterion is demonstrated to be a \gry for the $T_{pw}$ \gry theory. Thus, this modification of $k$-NN also corroborates the main conjecture. 

\subsection{Decision trees}

For this learner, the features  are expected to be ``ordinal'': every  feature has finite number of ordered  values; there are no operations on feature values.  The feedback of observations is binary.

 The learner starts with whole domain, split it in two subdomains by a value of some feature. Then, the procedure is repeated for every of the subdomains  until a subdomain called "leaf" is reached. The decision is selected for this subdomain. The navigation over the tree of subdomains continues until some stopping criterion is reached. 
 The algorithm has a precise rule for generating the parameters of the next subdomain based on the previous trajectory and the obtained results.

 There are two criteria of  a leaf:
 \begin{enumerate}
 	\item Number of observations  in the subdomain is below a threshold $N$.
 	\item Percentage of observations  of the prevalent class in the subdomain is above the threshold $q.$
 \end{enumerate}

The procedure may be described as a wrapper algorithm:
\begin{tcolorbox}[title=Decision Tree: ]
\begin{itemize}
	\item \textbf{Generating parameters} $g$ of the next subdomain

		\begin{itemize}
			\item \textbf{Proper training:}
			\begin{itemize}
		\item \textbf{Focusing:} select subdomain $G(g)$ with parameters $g$
		\item \textbf{Fitting}: evaluate error rate of constant hypotheses $ \{0, \; 1\}$ in $G(g)$ 
		\item \textbf{Optimal selection}: if the leaf criteria  in $G(g)$ are satisfied, select the hypothesis with minimal error rate. 
	\end{itemize} 	
\end{itemize} 
 \item \textbf{End of loop} by parameters $g$		
\end{itemize}

\end{tcolorbox}

 In this case, we do not need to analyze whole tree before we create a wrapper \de: the decision is prevalent value on each leaf.  For the points, which do not belong to any leaf, the decision is not defined. 
 
 The error of a constant function in a subdomain  $Q$ is defined as empiric risk. And we demonstrated that empiric risk is total proper \gry of the point-wise \gry theory $T_{pw}$. 
Therefore,  this description of the procedure agrees with the main conjecture.

\subsection{Naive Bayes}

The algorithm works as if it deals with  nominal  data: the only relationship between data points is equivalence of feature values. The feedback of observations is binary, and so is feedback of the solution.

The procedure defines decision function on one data point at the time. 
For a given data point $z= \langle z_1, \ldots, z_n \rangle$ the procedure selects $n$ subsets of the training set.  Subset $S_j$ includes all the observations with $j$-th coordinate of the data point equal $z_j.$ For each subset $S_j$, the learner evaluates error rate $e_{j, c}$    of each hypothesis $c \in \{0, 1\}.$ Then for each hypothesis   it calculates loss 
$$\Delta(c, S)  =  1 - \prod_j ( 1- e_{j, c}).$$ The learner  selects a hypothesis  with the minimal loss as a decision. 

Let us define the \gry theory $T_{nb}$ for this learner.
LOH Language of the theory has an additional finite sort $\bb{N}$ with values $1, 2, ..., n$, and variables $i, j, i_1, \ldots, $ where $n$ is the dimensionality of the metric space for the sort $\bb{X}$. The language has an additional first order function $\bb{X} \times \bb{N} \rightarrow \bb{R}$, denoted $(x)_i, $ the $i$ coordinate of a vector $x \in \bb{R}.$

The theory $T_{nb}$  has $n$  aspects. For aspect $i \in  1:n$ the collision condition is 
$$\pi_i(\alpha_1, \alpha_2) =  (\bm{s}(\alpha_1) = \asymp) \; \& \;  (\bm{s}(\alpha_2) = \approx) \; \& \;  \Big( \big( \bm{x}(\alpha_1) \big)_i =  \big(  \bm{x}(\alpha_2) \big)_i \Big) . $$

The deviation function is the same for every aspect
$$\delta_i(\alpha_1, \alpha_2) = |\bm{y}(\alpha_1) -  \bm{y}( \alpha_2)| .$$

To properly aggregate all  deviations of an aspect of \gry  we use error rate: 
$$e_{j,c} = \frac{1}{k_j} \sum \delta_j\big(\alpha_1,  \alpha_2 \big), $$
where $k_j$ is the number of deviations for the $j$-th aspect in the full model. 

To combine aspect incongruities $\{e_{j,c}\}$ into total proper \grym we use the aggregation operation 
$$\Delta(c, S)  =  1 - \prod_i ( 1- e_{j,c}).$$ The function is isotone. 

We can conclude that the function $\Delta(c, S) $  satisfies the requirements on total proper aggregation.  At the same time it is  the loss criterion of this learner used to select the hypothesis with minimal value of this criterion.

Now the procedure of the learner may be described as very simple

\begin{tcolorbox}[title = Naive Bayes]
	\begin{itemize}
	\item \textbf{Fitting}: generating hypotheses $c \in \{0, 1\}$ and calculating the loss criterion $\Delta(c, S):$ 
	\begin{itemize}
		\item Loop by aspects $ i \in 1:n$
		\begin{itemize}
			\item Calculating error rate $e_{i,c}$ for the aspect $i$
		\end{itemize}
	\item Calculating the criterion $\Delta(c, S)$
		
	\end{itemize}
	\item \textbf{Optimal Selection}: selection of the hypothesis $c$  with the lowest criterion $\Delta(c, S). $ 
	\end{itemize} 
\end{tcolorbox}

This proves that Naive Bayes supports the main conjecture.

A product in the aggregation of the aspects in the loss function  is chosen  in Naive Bayes   because  it is  sensitive to the low frequencies of class: if some  value $1 - e_{i, c} $ is  close to 0, the product will be affected much more than the sum of the frequencies, for example. If some feature value almost never happens in a given class $c$, the hypothesis $c$ will have no chance of being selected, regardless of other feature values of $z$.  It justifies choice of product for aggregation.

The products of frequencies  are  traditionally interpreted as evaluation of posterior probabilities with ``naive''  assumption that the feature values are independent. There are several issues with this narrative. The first is its uniqueness. Only this learner is based on  Bayes rule. Other learners would need different foundations.   
Another issue is that  it creates an impression that the learner needs  an improvement,  is not sophisticated enough. It means, the narrative does not, really, explain or justify this learner.

I hope, I demonstrated that interpretation of the learner as   ``naive'' and ``Bayesian'' misses the point. The procedure   is driven by its specific data type, not by naive fondness for  Bayes theorem.  

\subsection{Logistic Regression}

This learner assumes the features are continuous, the feedback of the observations is binary, but the feedback of the decision is continuous. The decision is defined on the domain  $\chi. $ The procedure of generating the hypotheses is not specified. 

The class of functions associated with logistic regression is 
$$F =\left \{ \frac{ 1} {1 +  exp(- \langle w, x\rangle)  } \right \}.$$
The functions have values in the interval $(0, 1).$

The learner minimizes criterion

$$\Delta(f, S) = \frac{1}{m} \sum_{s \in S} \log\Big( |y(s) - f\big(x(s) \big)|  \Big).$$

Let us define the \gry theory $T_{lr}$ for this learner.
There is one collision condition which coincides with the condition $\pi_{pw}$ of $T_{pw}.$
$$\pi(\alpha_1, \alpha_2) = (\bm{s}(\alpha_1) = \asymp) \; \& \; ( \bm{s}(\alpha_1) = \approx)
\; \& \; (\bm{x}(\alpha_1) = \bm{x}(\alpha_2)).$$

The deviation function is
$$\delta(\alpha_1, \alpha_2) = log( \rho_y(\alpha_1, \alpha_2)).$$
The  aggregation  uses recursive aggregation functions from  the line 1 of the table (\ref{tab:aggregations}). 
Thus the loss criterion $\Delta(f, S)$ coincides with total proper  \gry  for the training set $S,$ hypothesis $f$ and the \gry theory $T_{lr}.$

So, the logistic regression supports the main conjecture as well.

\subsection{Linear SVM for classification}

All the previous learners belong to  machine learning ``folklore''. Their authors are not known, or, at least,  not  famous. 

SVM is one of the first learners associated with a known author: it  is  invented by  V. Vapnik. His earliest English publications  on this subject appeared in early nineties \cite{Vapnik1}, \cite{Vapnik2}.

Let us start with linear SVM for binary classification. 
The observations $$S = \{\beta_i, i= 1:m\}$$  have two class labels: $\{-1, 1\}$  with data points $x \in \bb{R}^n. $

The class of hypotheses $F$ consists of  linear functions $f(x) $ with $n$ variables.  For a $f \in F,  f(x) = x^T \beta  + \beta_0.$ denote $\bm{w}(f) = \beta, \bm{b}(f) = \beta_0.$

The problem is formulated as minimization of the criterion
\begin{tcolorbox}[title = Linear SVM]
\begin{align} \label{SVM}
	L(f, S, \xi)  =   \alpha \, \| \bm{w}(f)  \|^2 + \frac{1}{m} \sum_{\beta \in S }^m \xi(\beta) \\
 \text{s.t. }  \text{ for all }  \beta \in S,  \; \; \bm{y}(\beta) \cdot f(\bm{x}(\beta)) \ge 1 - \xi(\beta)\; \text{ and } \; \xi(\beta)  \ge 0. \label{conditions} 
\end{align}
\end{tcolorbox}

The criterion may be simplified though. For this, we want to switch to narrower class of functions, which shall contain all the same decisions. 

The observations $\beta \in S$ satisfying condition $\bm{y}(\beta) \cdot f(\bm{x}(\beta)) > 0. $  are considered correctly classified by the function $f$. Denote $S^\oplus(f)$ all correctly classified observations by the function $f, $ and $S^\ominus(f) = S \setminus S^\oplus(f)$ the rest of the observations. 

Let us consider all the functions $f \in F$ such that 
$ S^\oplus(f) \neq \emptyset $ and 
$$\min_{S^\oplus(f) }| f(\bm{x}(\beta))| = 1.$$
Denote this class of function $F^\prime(S).$  The class $F^\prime(S)$ is not empty. Indeed, if for some $f, f  \not \equiv 0,$ $S^\oplus(f) = \emptyset$, then, $S^\oplus( -f) = S$.  If  $$q = \min_{S^\oplus(f) }| f(\bm{x}(\beta))| \neq  1, $$ then 
the function $f^\prime = \frac{1}{q} f$ satisfies the condition $$\min_{S^\oplus(f) }| f^\prime(\bm{x}(\beta))| = 1.$$

The last consideration implies that if $f$ is the decision of the problem, 
then the problem has a decision  $f^\prime$ in the class $F^\prime(S)$ with the same set of correctly recognized observations $S^\oplus(f^\prime) = S^\oplus(f).$

Therefore, we can restrict the search for a decision in the class $F^\prime(S)$ only.

\begin{theorem}
	The linear  SVM classification problem minimizes the loss criterion
	$$L_{svm}(f, S) = \alpha \| \bm{w}(f) \|^2 + \frac{1}{m} \sum_{\beta \in S^\ominus(f)} |\bm{y}(\beta) - 	f(\bm{x}(\beta)) |,$$
for $f \in F^\prime(S). $
\end{theorem}
\begin{proof}

 The conditions (\ref{conditions})  can be rewritten as $\forall \beta, \beta \in S:$
 \begin{equation}\label{cond}
 	\begin{cases}
 		\xi(\beta) \geq  1  - \bm{y}(\beta) \cdot f(\bm{x}(\beta))   \\
 		\xi(\beta) \geq 0.
 	\end{cases}
 \end{equation}
 or  
 $$\xi(\beta) \ge \max \big\{  1 - \bm{y}(\beta) \cdot f(\bm{x}(\beta)) , \; 0 \big\}.$$

  The values $\xi(\beta), \beta \in S$  do not depend on each other, so the minimum of their sum  is achieved when every variable $\xi(\beta)$ equals its lowest possible value. Let us find these lowest values for $\xi(\beta)$ depending on  if $\beta \in  S^\oplus(f)$ or  $\beta \in S^\ominus(f).$
 
 If $\beta \in S^\oplus(f),$ 
 $$\bm{y}(\beta) \cdot f(\bm{x}(\beta)) = |f(x(s))|. $$ 
 
 By definition of $F^\prime(S), $ $|f(x(s))| \geq 1.$
 Then $$\xi(\beta) \geq \max \big\{  1 - \bm{y}(\beta)\cdot f(\bm{x}(\beta)) , \; 0 \big\} = 0.$$ 
 In this case, the lowest possible value for $\xi(\beta)$ is 0.

If $\beta \in S^\ominus (f),$ $$\bm{y}(\beta) \cdot f(\bm{x}(\beta)) = - |f(\bm{x}(\beta))|.$$ Then 
$$\xi(\beta) \geq \max \big\{  1 - \bm{y}(\beta)\cdot f(\bm{x}(\beta)) , \; 0 \big\} = 1 + |f(\bm{x}(\beta))|.$$
In this case, the lowest possible value for $\xi(\beta)$ is $1 + |f(\bm{x}(\beta))|.$

So,

\begin{equation}
\min_\xi \frac{1}{m} \sum_S \xi(\beta) = \sum_{\beta \in S^\ominus(f)} ( 1 + |f(\bm{x}(\beta))| ).
\end{equation}

We still need to prove that for $\beta \in S^\ominus(f)$
$$1 + |f(\bm{x}(\beta))|  =  |\bm{y}(\beta) - f(\bm{x}(\beta))|. $$

Let us take $\beta \in S^\ominus(f).$ 
If $\bm{y}(\beta) = 1, $ then $f(\bm{x}(\beta)) < 0$ and $|f(\bm{x}(\beta))| = - f(\bm{x}(\beta)).$ So, 
$$(1 + |f(\bm{x}(\beta))| ) = 1 - f(\bm{x}(\beta)) = | \bm{y}(\beta) - f(\bm{x}(\beta)) |.$$
If $\bm{y}(\beta) = -1, $ then $f(\bm{x}(\beta)) > 0$ and $|f(\bm{x}(\beta))| = f(\bm{x}(\beta)).$ So,
$$1 + |f(\bm{x}(\beta))| = 1 + f(\bm{x}(\beta)) = -\bm{y}(\beta) + f(\bm{x}(\beta)) = |\bm{y}(\beta) - f(\bm{x}(\beta))|.$$

\end{proof}

The part $\|w(f)\|^2$ of the criterion  is a regularization component: $w(f)$ is the gradient of the hypothesis $f$, and  $\|w(f)\|^2$ is the square of its norm. Minimizing this component, we reduce the speed of the hypothesis change  and make the model more ``predictable''. 

Now  to prove that the learner agrees with the main conjecture,  I just need to define the \gry theory which explains the second component of the loss criterion 
$$L(f, S) = \frac{1}{m} \sum_{\beta \in S^\ominus} |\bm{y}(\beta) - f( \bm{x}(\beta))\|.$$

For this, we need to define how the distance is measured between the feedback of observations and the function value.

The rule is: for $\alpha_1 = (\asymp(\ph(x_1) = y_1) ), \alpha_2 = (\approx( \ph(x_2) = y_2))$
\[\rho_y(\alpha_1, \alpha_2) = 
\begin{cases}
	0, & \text{ if } y_1 \cdot y_2 \ge 0\\
	|y_2 - y_1|, & \text{otherwise}. 
\end{cases}
\]

Then the  \gry theory $T_{svm}$  coincides with point-wise \gry theory $T_{pw}$. The total proper \gry  is constructed using  proper recursive aggregation defined in the first line of the table \ref{tab:aggregations}.

\subsection{Linear Support vector regression}

The learner minimizes criterion \cite{Hastie}

$$L_{svr}(f, S) = \sum_{i = 1}^m V_\epsilon\big(\bm{y}(\beta_i) - f(\bm{x}(\beta_i))\big) + \lambda \|\bm{w}(f)\|^2, $$
where 
\[
V_\epsilon(r) = 
\begin{cases}
	0, & \text{if  } |r| < \epsilon\\
	|r| - \epsilon, & \text{otherwise}
\end{cases}
\]
and $S = \{\beta_1, \ldots, \beta_m\}.$

The second component of the loss criterion is regularization, the same as in the SVM.

The distance between feedback of an observation and the value of a hypothesis is defined through the function $V:$
for $\alpha_1 = (\asymp(\ph(x_1) = y_1) ), \alpha_2 = (\approx( \ph(x_2) = y_2))$
$$\rho_y(\alpha_1, \alpha_2) = V(y_1  -  y_2).$$

Then the \gry theory $T_{svr}$ coincides with the point wise theory $T_{pw}.$
The total proper aggregation is defined again as the first line in the table \ref{tab:aggregations}. 

So, the linear support vector regression supports the main conjecture as well.

\subsection{Support Vector Regression with Kernels}

Suppose (\cite{Hastie}) we have a set of basis functions $H= \{h_i(x), i = 1,\ldots, k\}.$
We are looking for hypotheses 
$$f(x) = \sum_{i=1}^k w_i h_i(x)  + b.$$
The loss criterion  used here is
$$L(f, S) = \sum_{i= 1}^m V\big( \bm{y}(\beta_i) - f(\bm{x}(\beta_i) \big)  + \lambda \|\bm{w}(f)\|^2,$$
where 
\[
V(r) = 
\begin{cases}
	0, & \text{if  } |r| < \epsilon \\
	|r| - \epsilon, & \text{otherwise}.
\end{cases}
\]

Here the transformation  $x \rightarrow \langle h_1(x), \ldots, h_k(x)\rangle$ from a $n$-dimensional space $R^n$ into $k$-dimensional space $H(x)$ may be called focusing.
Then the problem is reduced to solving a linear SVM regression in the transformed space. 
Thus, SVR with kernel supports the main conjecture as well.

\subsection{Ridge Regression} 

The learner  finds the solution in the same class of linear hyperplanes $F = \{f: \; \, f= wx + b\} $ as linear  SVM  for classification, and it has the criterion

$$ L_{rr}(f, S) =  \alpha \|\bm{w}(f)\|^2 + \frac{1}{m} \sum_{\beta \in S}  (f(\bm{x}(\beta)) 	- \bm{y}(\beta))^2.$$

The first component of the loss criterion is regularization, the same as in SVM, SVR. 

The second component can be explained as total proper \gry where the theory's  only collision condition  coincides with the condition $\pi_{pw}$ of point-wise \gry theory $T_{pw}$, the  deviation is defined by the formula
$\delta(\alpha_1, \delta_2) = (\bm{y}(\alpha_1) - \bm{y}(\alpha_2))^2. $
and the recursive aggregation is defined in the line 1 of the table \ref{tab:aggregations}.

Thus,  Ridge regression corroborates the main conjecture too.

\subsection{Neural Network  (NN) }

Let us consider single hidden layer NN for two class classification as it is described in \cite{Hastie}.

First, the learner transforms  $n-$ dimensional metric space of inputs $\bb{R}$ into $k$-dimensional space $\bb{Z}$ using non-linear transformation;
$$Z_i(x) = \delta( g_i(x) ), i = 1, \ldots, k,$$
where $\delta(r)$ is delta function and $g_i$ are linear functions.  Denote $\bm{z}(x)$ the vector with coordinates $\langle Z_1(x), \ldots, Z_k(x) \rangle.$

Then, for each class $c \in \{0, 1\}$, the learner builds  linear voting function  $f_c(\bm{z}(x)).$ 
Denote $G = \{g_1, \ldots, g_k\}$, and $F= \{f_0, f_1\}.$ 

For each $x \in \bb{R}^n$ the class is selected as
$C(x,  G, F) =  \arg \max_c  f_c(\bm{z}(x)).$ 

The learner uses the  loss criterion 
$$L_{nn}(G, F, S) =  \sum_{\beta \in S} ( \bm{y}(\beta) - C(\bm{x}(\beta), G, F) ).$$
It is obvious that he loss criterion is a total proper \gry for  by  the point-wise \gry theory. 

The learner optimizes simultaneously  parameters of the functions $G$ and $F.$ For selection of parameters of these functions the learner uses gradient descent, which is called ``back propagation'' in this case. The learner uses some additional stopping criterion.  

So, the procedure does not have a focusing stage. If calculates loss for given set of parameters,  evaluates  gradients by  each parameter, and then updates parameters based on the gradients. After the stopping criterion is achieved, the algorithm outputs the decision with the lowest loss criterion. 

The procedure has only two types of steps: 
\begin{enumerate}
\item fitting, which includes
\begin{itemize} 
\item  generation of the $C(x, G, F)$ hypothesis based on previous value of loss criterion and gradients 
\item evaluation of loss criterion of the current hypothesis $L_{nn}(G, F, S).$
\end{itemize}
\item optimal selection: selection of the hypothesis with the lowest loss criterion.  
\end{enumerate}

Thus, NN also agrees with the main conjecture.

\subsection{$K$ Means Clustering}

The learner is different from hierarchical clustering in that it does not combine clusters, rather, for each observation, it chooses the proper cluster. It is assumed that the distance on the domain of  data points is Euclidean. 

Here is the description of the learner from \cite{Hastie}. 
	\begin{enumerate}
		\item Given the current set of means of clusters $M = \{m_1, \ldots, m_K\}$, each observation is assigned to the cluster with the closest mean.
		\item The rounds of assignment of all observations are repeated until clusters do not change.
	\end{enumerate}

The proper learning happens when we search for the cluster for the given observed data point.  Denote $C(x)$ the assignment of a cluster to a data point $x$. Given the set of observed  data points $S_x = \{x_1, \ldots, x_m\},$  $K$ clusters  with cluster centers $M$ of the sizes $\{l_1, \ldots, l_K\}$
the procedure assigns a new class to an observed data point  to  minimize sum of all pairwise distances within each cluster
\begin{align}
W(C, S) & = \frac{1}{2} \sum_{k=1}^K \sum_{C(\xi) = k} \sum_{C(\zeta) = k} \| \xi - \zeta\|^2 \label{1line}\\ 
        & = \sum_{k=1}^K l_k \sum_{C(\xi) = k}  \| \xi - \overline{x}_k\|^2,
\end{align}
where $\xi, \zeta \in S$, and $\overline{x}_k$ is mean of the $k$-th cluster. 
I use the (\ref{1line}) to prove that the learner agrees with the main conjecture. 

Denote  $x_0$ a data point, $x_0 \in S_x,$ which we need to assign a cluster on this step.

As in the case of hierarchical clustering,  we consider underlying dependence $\ph$ as a function from cluster index $k$ to the observed data point $x.$  There are $K$ hypotheses $H(x_0) = \{h_1, \ldots, h_K\}.$ Each hypothesis $h_i$ has a single hypothetical case $\asymp( \ph(i) = x_0).$ 

 We assume, before current run of the learner,  the clusters are already assigned to each observed data point besides $x_0$. So, the run starts with the training set having observations $$S = \{ \big(\approx( \ph( i) = x)\big)\;  |  \; i \in 1, \ldots, K;\;  x \in S_x \setminus \{x_0\}\}.$$

There is only one aspect \gry with the collision condition
$$\pi(\alpha_1, \alpha_2) = ( \bm{x}(\alpha_1) = \bm{x}(\alpha_2)),$$
which says that we evaluate deviation for each pair of formulas with the same argument $x$, the same cluster, regardless of modality. The collision condition is symmetrical, therefore for each pair of formulas $\alpha_1, \alpha_2$  which satisfies the condition, the pair $\alpha_2, \alpha_1$ satisfies the condition as well. In effect, every pair is counted twice.  

The deviation function is
$$\delta(\alpha_1, \alpha_2) = \rho_y(\alpha_1, \alpha_2)^2 $$

For the proper aggregation we use the averaging. The formula of loss criterion $W(C, S)$ does not  explicitly have the scaling coefficient $\frac{1}{m}$ because it would be the same for every hypothesis. Otherwise, the loss criterion in this case is the total proper \gry for the described theory. 

The learner generates all hypotheses $H(x_0)$, evaluates the loss criterion for each of them and selects the hypothesis with the lowest loss criterion. Thus this learner corroborates the main conjecture as well.

\section{Conclusions}

Here I propose a modal logic  LOH to explain the learning in  machine learning.
The logic generalizes existing learners to explain, what we do, when we learn.

The underlying dependence we are learning is assumed to be non-deterministic.  The first order formulas of LOH (Logic of hypotheses and observations) are  statements about values of the underlying dependence in some points.  The formulas   always have modalities ``it appears''  or ``assume that':  they  describe observations  and  hypotheses respectively. Being subjective, modal formulas can  not have  truth values,  so they can not have   contradictions, inconsistencies in the strict logical sense. 

The underlying dependence  is expected  to be  predictable in the  vague sense that  ``close'' data points shall correspond to ``close'' feedback. The implied   ``closeness'' depends  on the task: its data types, precision of measurement, goals and so on.     Instead of the predictability, I formalize the opposite concept:  \gryp  It is defined to be flexible to match the tasks  as well.  Each version of  \gry is defined by  its own ``\gry theory'',  where ``collision conditions'' are binary  predicates expressed as second order formulas  of LOH, and the ``deviation''   functions evaluate disagreement between ``colliding''  first order formulas.

The main conjecture of this work is that every learner has a loss criterion which can be presented  as \gry in some  \gry theory and, given the observations, the learner  performs certain steps to find the hypothesis minimizing this loss criterion. 

The  main conjecture is illustrated on large number of popular learners, including SVM,  SVR, hierarchical clustering, $K$ mean clustering, neural network, Naive Bayes and others. 
Each of these learners corroborates the main conjecture. 

Here are some of the advantages of the proposed ML paradigm over traditional statistical one.
\begin{enumerate}
	\item  The framework  provides unified logical justification and explanation for large variety for real life learners used by practitioners. It   explains how and why we can learn from fixed finite data, while statistical learning theory is not able to do it. 
	
	\item  I demonstrated  inner similarity of the regression, classification and clustering methods: all of them are shown to corroborate the main conjecture. Statistical learning theory  can not include  clustering in their concept of learning with ever increasing training set. 

    \item The proposed approach  allows to understand ``regularization'' component of loss criteria as an aspect of \gryp 
    
	\item  Described here a general structure of  a learner  shall facilitate  classification, selection, customization and design of  new learners.   The proposed language can express much wider variety of learners than  are being commonly used. New varieties of  learners may be especially advantageous when the available data are limited.  
	
	\item As an example of such learner customization, I proposed a  version of adaptive $k$-NN learner based on Hoeffding inequality. The learner shall have advantages over ADA $k$-NN for small data.  
	
	\item In addition, the concept of \gry is demonstrated to be helpful for some of common data analysis problems, where statistics approach appears to be inadequate also. 
\end{enumerate}


 I want to point out philosophical implications of these results. The learning is usually considered to be an inductive process: they  say, the decision  ``generalizes'' observations. Philosophers \cite{Popper} noticed logical contradictions of the concept of induction:  how can  a decision  ``follow'' from the data? 

The main conjecture suggests a possible explanation. Suppose, a class of hypotheses is fixed, and  we need a certain  type of agreement between observations  and a hypothesis.    We also believe that the world is somewhat predictable: usually, it does not change sharply. Experience of our and other species teaches us that, otherwise we could not survive  in a rapidly changing, unpredictable environment. This belief is called  here fundamental. 

If we rely on the fundamental belief,  the winning strategy is to find the hypothesis, which violates the agreement with the observations the least: it has the best chance to be good in the future too. And this is exactly what we do in machine learning, minimizing the \gryp 

If the class of hypotheses is infinite, we may not find the optimal decision,  we may only  approximate it. 

 Thus, the main conjecture shows that   learning  (and induction)  work as  ``approximate deduction''.  The (approximate)  decision is ``deducted'' from the \gry theory and the observations. 
 
We believe in predictability, knowing that the dependencies we learn are non-deterministic. This makes testing a critical part of learning cycle. I plan to talk about it in the following work.

\bibliographystyle{plain} 
\bibliography{Smoothing} 

\end{document}